\documentclass[numsec,webpdf,modern,medium,namedate]{JRSSB}
\onecolumn

\usepackage[doublespacing]{setspace}
\usepackage[fontsize=12pt]{fontsize}

\usepackage{amssymb}
\usepackage{makecell}
\usepackage{array} 
\usepackage[]{natbib}
\usepackage{bbm}
\usepackage{amsmath,amsfonts,bm}
\usepackage{parskip}
\usepackage{amsthm}

\usepackage{microtype}
\usepackage{graphicx}
\usepackage{booktabs} 
\usepackage{multirow} 
\usepackage{comment}
\usepackage{caption}
\usepackage{bbm}
\usepackage{subcaption}

\usepackage{tikz}

\usetikzlibrary{shapes,decorations,arrows,calc,arrows.meta,fit,positioning}
\tikzset{
    -Latex,auto,node distance =1 cm and 1 cm,semithick,
    state/.style ={ellipse, draw, minimum width = 0.7 cm},
    point/.style = {circle, draw, inner sep=0.04cm,fill,node contents={}},
    bidirected/.style={Latex-Latex,dashed},
    el/.style = {inner sep=2pt, align=right, sloped}
}

\usepackage{algorithm}
\usepackage{chngcntr} 

\DeclareMathOperator*{\expit}{expit}

\DeclareMathOperator*{\argmin}{arg\,min}
\DeclareMathOperator*{\arginf}{arg\,inf}
\DeclareMathOperator*{\argmax}{arg\,max}

\theoremstyle{definition}
\newtheorem{theorem}{Theorem}
\newtheorem{prop}{Proposition}
\newtheorem{lemma}{Lemma}
\newtheorem{defi}{Definition}
\newtheorem{remark}{Remark}
\newtheorem{assumption}{Assumption}[section]

\newtheorem{example}{Example}
\usepackage{dsfont}
\usepackage{colortbl}
\definecolor{lightblue}{rgb}{0.8, 0.9, 1}
\definecolor{lightred}{rgb}{1, 0.8, 0.8}

\usepackage{threeparttable}

\newcommand{\RNum}[1]{\uppercase\expandafter{\romannumeral #1\relax}}

\newcommand{\E}{\mathbb{E}}

\definecolor{green}{RGB}{20, 171, 132}

\usepackage{hyperref}

\hypersetup{
	unicode=false,
	pdftoolbar=true,
	pdfmenubar=true,
	pdffitwindow=false,     
	pdfstartview={FitH},    
	pdfkeywords={}, 
	pdfnewwindow=true,      
	colorlinks=true,       
	linkcolor=cyan,          
	citecolor=blue,        
	filecolor=pink,      
	urlcolor=cyan           
}

\usepackage[utf8]{inputenc}

\newcommand{\neighbor}[1]%
{\overline{#1}}

\usepackage{mathtools}
\graphicspath{{fig/}}
\DeclarePairedDelimiter\norm{\lVert}{\rVert}

\allowdisplaybreaks

\begin{document}

   \journaltitle{Paper}
\DOI{}
\copyrightyear{2026}
\pubyear{}
\access{}
\appnotes{}

\firstpage{1}

\title{ Double Fairness Policy Learning: Integrating Action Fairness and Outcome Fairness in Decision-making}

\author[1,2]{Zeyu Bian}
\author[2,$\ast$]{Lan Wang}
\author[3]{Chengchun Shi}
\author[4]{Zhengling Qi}

\authormark{Bian et al.}

\address[1]{\orgdiv{Department of Statistics}, \orgname{Florida State University}, \orgaddress{Tallahassee, FL, USA}}
\address[2]{\orgdiv{Department of Management Science}, \orgname{University of Miami}, \orgaddress{Miami}, \state{FL}, USA}
\address[3]{\orgdiv{Department of Statistics}, \orgname{London School of Economics and Political Science}, \orgaddress{London, UK}}
\address[4]{\orgdiv{Department of Decision Sciences}, \orgname{George Washington University}, \orgaddress{Washington, DC, USA}}


\corresp[$\ast$]{Lan Wang, Department of Management Science, University of Miami, Miami, FL 33146, USA. lanwang@mbs.miami.edu}

\abstract{
Fairness has recently received significant attention in machine learning, driven by ethical concerns surrounding algorithmic applications.
However, fairness in policy learning, where the goal is to make fair decision while simultaneously maximizing expected rewards, has been less explored. 
Most existing works in fairness for policy learning focus primarily on action fairness. However, many applications require fairness in both the actions taken and their resulting outcomes. This paper tackles the challenge of double fairness in policy learning by considering the trade-off of three objectives simultaneously: action fairness, outcome fairness, and maximizing the value function. We propose a novel framework that integrates fairness directly into a multi-objective optimization problem for policy learning. Our framework is flexible and accommodates various commonly used fairness notions.  We establish theoretical guarantees for the proposed method and demonstrate its empirical superiority over competing approaches. Finally, we apply our method to a motor third-party liability insurance dataset and an entrepreneurship training dataset. Our proposed method demonstrates greater fairness on both fairness metrics while having only a modest reduction in overall revenue.}


\maketitle

\section{Introduction} \label{sec:intro}

As a core aspect of trustworthiness,
fairness has become central to machine learning \citep{dwork2012fairness,zemel2013learning, grgic2016case,hardt2016equality, kusner2017counterfactual, wu2019counterfactual,tan2022rise,chen2023learning}, 
motivated by ethical concerns about the societal impacts of algorithmic decision-making. 
In many domains such as healthcare, insurance, lending, and public programs, algorithmic systems can affect large populations at scale. Purely accuracy- or profit-driven optimization can amplify disparities, imposing disproportionate burdens on vulnerable groups, and exposing organizations to legal and reputational risk. Fairness is therefore not merely a normative ideal. It is an operational requirement for responsible deployment and a critical pillar of ethical AI.


Most existing work on algorithmic fairness has focused on supervised learning. In contrast, fairness in  policy learning, where the goal is to  maximize expected reward or social welfare, has received comparatively limited attention \citep{nabi2019learning, tan2022rise,viviano2023fair,fang2023fairness, wang2025counterfactually}. A key distinction is that policy learning is inherently interventional: a policy selects an action $A$, and that action 
alters subsequent outcomes. This leads to two conceptually distinct fairness targets. Action fairness concerns whether the policy treats individuals similarly in its action choices across sensitive groups (e.g., equalized treatment assignments). Outcome fairness instead asks whether the consequences (e.g., the benefits, welfare, or other fairness-relevant outcomes) induced by the policy are equitably distributed across groups. Importantly, equalizing actions does not, by itself, equalize outcomes when groups face different constraints, environments, or downstream responses to the same action.
Consistent with this distinction, our analysis of a Belgian motor liability insurance dataset
\citep{dutang2019casdatasets}
show that standard policy optimization methods
can induce substantial disparities in both action and outcome fairness (Table~\ref{tab:ins} in
Section~\ref{sec:real data}).

This ``action--outcome'' gap arises naturally in practice.
For instance, a government job training program might admit participants
equally (action fair), 
unequal downstream constraints such as transportation access, caregiving burdens, or financial instability can still prevent some participants from completing the program or benefiting from it to the same extent.
This could result in an inequitable final outcome (e.g., different job placement
rates) even though the admission decision appears fair.
In our analysis of such a training program in Section \ref{sec:entre}, we
found a stark example of this disparity: one demographic group had a nearly
20\% lower probability of securing an internship than another.
This underscores the importance of outcome-level fairness in
decision-making.

This work is also motivated by growing concerns about fairness in the insurance industry. Machine learning algorithms are increasingly used to automate risk assessment, pricing, and claims-related decisions, making trustworthy decision rules especially critical. For example, UnitedHealth has used the nH Predict platform \citep{talia2024algorithms} to assess patient risk and forecast outcomes in ways that inform coverage and pricing. While such tools can improve efficiency, they also raise concerns about algorithmic bias and discrimination against vulnerable populations. A recent class-action lawsuit alleges that the nH Predict algorithm was used to unlawfully deny or curtail rehabilitation care for seriously ill and elderly.\footnote{See the news report at \href{https://www.healthcarefinancenews.com/news/class-action-lawsuit-against-unitedhealths-ai-claim-denials-advances}{Healthcare Finance News}.} 

The insurance context further illustrates why action fairness alone is {it insufficient}.
In automobile insurance, a typical ``action'' is to classify a customer as
low-risk vs.\ high-risk, which 
directly affects premiums.
Action fairness focuses on whether similarly situated individuals across
groups receive comparable risk labels and charged comparable premiums.
However, even under perfect action fairness, two individuals with similar profiles receiving the same label and premium, their realized protection can still differ across groups.
For example, claims may be approved at different rates, reimbursed
differently, or processed with different delays across groups due to
differences in documentation burdens, access to provider networks, or other
downstream mechanisms.
In such cases, the action is equal, but the policy-induced outcome, the
expected net benefit or welfare a customer receives, remains unequal.
Outcome fairness is designed to detect and mitigate precisely these
disparities.

These findings highlight the need to balance the decision maker's value
objectives with fair treatment across different groups. Motivated by these concerns, we propose a novel double fairness policy learning framework that jointly targets 
(a) action fairness;
(b) outcome fairness;  and
(c) maximization of the overall value function.
We refer to this integrated goal as
double fairness learning (DFL). 

The DFL framework raises three central challenges. First, outcome fairness can only be controlled indirectly through the policy’s actions and the environment’s response: unlike action fairness, one cannot directly “intervene” on outcomes, so it is unclear when an outcome-fair (or doubly fair) policy exists at all. Second, the three objectives, action fairness, outcome fairness and value maximization, typically conflict, so it is rare for a single policy to optimize all three simultaneously. Hence, a principled framework is needed to characterize and manage these trade-offs, for which we adopt the concept of Pareto optimality. Third, computing the relevant Pareto fairness set is nontrivial in practice: standard linear scalarization approaches can miss Pareto-optimal solutions when the policy class is non-convex, as is typical for many policy parameterizations \citep{schulman2015trust}.

We address these challenges with a unified methodology and theory. We first provide a systematic analysis of when outcome-fair and doubly fair policies can be achieved, yielding interpretable conditions for feasibility. To manage trade-offs, we cast DFL as a multi-objective optimization problem \citep{miettinen1999nonlinear} and target policies that are Pareto-efficient with respect to the two fairness criteria while achieving high value function. To compute Pareto solutions, we adopt a lexicographic weighted Tchebyshev scalarization strategy \citep{steuer1983interactive}, which can recover Pareto solutions without requiring convexity (see Figure~\ref{fig:ls_vs_tch} in Appendix). We then develop an estimator that embeds fairness objectives directly into the policy optimization procedure, establish regret bounds that provide rigorous performance guarantees under double fairness constraints.

The proposed DFL framework is flexible and accommodates a broad range of fairness notions
from the literature, making it applicable across domains where decisions have
heterogeneous downstream impacts.
Moreover, DFL offers a practical framework for internal auditing of algorithmic impacts on consumers, helping organizations surface fairness concerns before they escalate into legal or reputational risks. For instance, see the recent action taken by the \href{https://www.ftc.gov/news-events/news/press-releases/2024/07/ftc-issues-orders-eight-companies-seeking-information-surveillance-pricing}{Federal Trade Commission}.

\textbf{\textit{Related Literature}}.  Policy learning seeks to maximize expected reward by learning an optimal policy that selects the best action for each situation. It has broad applications in precision medicine \citep{Murphy,gest,qian2011performance,zhao2012estimating,wallace2015doubly,wang2018quantile,shi2018high,qi2020multi,cui2021semiparametric,bian2023variable}, economics \citep{manski2004statistical,kitagawa2018should,athey2021policy}, and business strategy \citep{den2015dynamic,javanmard2019dynamic,bian2024tale}. This literature largely focuses on value maximization, whereas our work augments policy learning with two complementary fairness targets—action fairness (fairness of decisions) and outcome fairness (fairness of the induced consequences)—and studies how to optimize value under both.

Work on fairness-aware policy learning is comparatively scarce and typically enforces a single fairness notion. \citet{nabi2019learning} formulate fair policy learning via causal mediation analysis, specifying impermissible causal pathways from the sensitive attribute and optimizing policies under the resulting causal constraints.  \citet{tan2022rise}, \citet{fang2023fairness}, and \citet{cui2025policy} develop quantile-optimal decision rules designed to protect disadvantaged groups by controlling tail outcomes. \citet{viviano2023fair} adopt a “first do no harm” principle by characterizing a welfare Pareto frontier and selecting the fairest allocation within that welfare-efficient set, with computation based on mixed-integer optimization. \citet{wang2025counterfactually} study counterfactual fairness in sequential decision making, proposing approaches tailored to reinforcement learning settings. In contrast to these works, we explicitly distinguish action fairness from outcome fairness and address the double fairness problem with different theoretical tools.  Figure~\ref{fig:radar} compares our method
to existing fairness-aware policy learning approaches. We group prior
single-fairness methods into two categories: action single fair (ASF)
methods, which enforce only action fairness, and outcome single fair
(OSF) methods, which enforce only outcome fairness. Figure~\ref{fig:radar}
shows that our proposed DFL framework achieves consistently strong performance
across all three objectives (action fairness, outcome fairness, and value)
with modest compromises in each. By contrast, competing approaches
typically improve one dimension at the expense of substantial degradation in
the others.


\textbf{\textit{Organization}}. In Section \ref{sec:pre}, we provide a brief overview of policy learning and introduce fairness notations for both action fairness and outcome fairness. Section \ref{sec:double fair} first presents the necessary and sufficient conditions under which a outcome fairness or double fairness policy can be achieved. The task is then formulated within a novel multi-objective optimization framework, and the corresponding estimation procedure is described. Section \ref{sec:theory} establishes the theoretical properties of the proposed estimator. Finally, in Sections \ref{sec:sims} and \ref{sec:real data}, we demonstrate the proposed approach using both synthetic and real datasets.

\begin{figure}[H] 
    \centering
\includegraphics[width=0.5\textwidth]{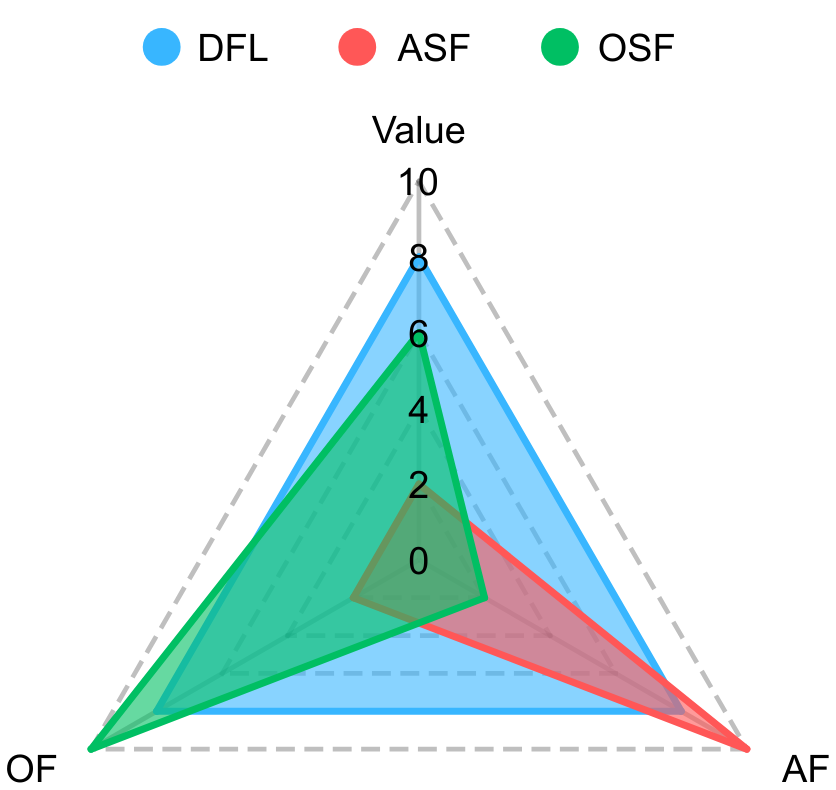}
    \caption{The performance of the three framework types is illustrated using a radar chart, where a higher score along each axis represents better performance in that specific aspect. Here, AF and OF refer to action fairness and outcome fairness, respectively. DFL refers to our proposed double fairness learning method. ASF represents the action single fair policy framework, and OSF denotes the value single fair policy framework.   } 
    \label{fig:radar}
\end{figure}

\section{Preliminary} \label{sec:pre}

 In line with many real-world applications, we consider a binary sensitive attribute and a binary action. Specifically, let \( S \in \{0,1\} \) denote the sensitive attribute, \( X \in \mathbb{R}^d \) a vector of non-sensitive covariates, \( A \in \{0,1\} \) the action, and \( R \) the resulting outcomes. The data-generating process (DGP) is summarized in Figure~\ref{fig: dag}.

Throughout, we allow for a general two-outcome formulation : \( R \equiv (R^{(1)},R^{(2)})^\top \in \mathbb{R}^2 \), where $R^{(1)}$ is the primary outcome that the decision maker seeks to maximize, while $R^{(2)}$ is an outcome along which fairness or equity considerations are assessed. 
This formulation accommodates a range of applications. In the CASdatasets (Section \ref{sec:insurance}), $R^{(1)}$ represents company revenue, while $R^{(2)}$ denotes customer welfare. We also accommodate the degenerate case where $R^{(1)} = R^{(2)}$: for instance, in the government job-training data (Section \ref{sec:entre}), both components correspond to subsequent entrepreneurial activity. In dynamic pricing, $R^{(1)}$ typically represents net profit, whereas $R^{(2)}$ might denote demand or consumer surplus—metrics for which fairness has been extensively studied \citep[see, e.g.,][]{cohen2022price, bian2026beyond}. Furthermore, in healthcare, $R^{(1)}$ may represent clinical outcomes, while $R^{(2)}$ denotes the allocation of medical resources.

We define $r(s,x,a) = \mathbb{E}\left(R^{(1)} \mid S = s, X = x, A = a\right)$ as the conditional expected primary outcome (the reward function), and $f(s,x,a) = \mathbb{E}\left(R^{(2)} \mid S = s, X = x, A = a\right)$ as the conditional expected fairness-related outcome.

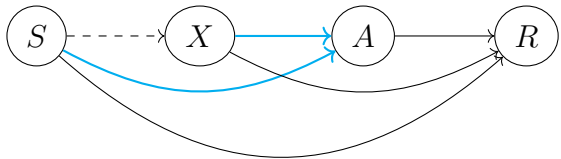
\begin{figure}[H] 
\centering 
\begin{tikzpicture}[xscale=1.2, yscale=2]

    \node[state] (s) at (0,0) {$S$};
    \node[state] (x) at (1.8,0) {$X$};
    \node[state] (a) at (3.6,0) {$A$};
    \node[state] (r) at (5.4,0) {$R$};

    \path[->, cyan, thick] (s) edge[bend right=18] (a);
    \path[->, dashed] (s) edge (x);
    \path[->, cyan, thick] (x) edge (a);
    
    \path[->] (a) edge (r);
    \path[->] (x) edge[bend right=18] (r);
    \path[->] (s) edge[bend right=28] (r);

\end{tikzpicture}
\caption{The directed acyclic graph of the DGP. The arrow in cyan means that the path is intervenable. The dashed line between $S$ and $X$ means that $S$ might influence $X$, but the relationship is uncertain or not definitively established. } 
\label{fig: dag}
\end{figure}

\subsection{Policy Learning}
We first provide a brief overview of the policy learning task. 
Recall that \( S \) denotes the sensitive attribute, \( X \) represents the non-sensitive covariates, \( A  \) is the action, and \( R  \) is the reward. 
A  \textbf{\textit{policy}} $\pi$ prescribes a strategy that selects the action based on  $X$ and $S$, which is a function mapping from the information into a probability distribution over the action space. The objective in policy learning is to find an optimal policy $ \pi^{\mathrm{opt}}$, such that the value (expected reward) is maximized, i.e., \vspace{-0.6cm}
\begin{gather*}
   \pi^{\mathrm{opt}} \in \argmax_{\pi \in \Pi} \left[ V(\pi)\equiv \E^\pi\left(R^{(1)}\right)\right],
\end{gather*}
where $\E^\pi$ denotes the expectation of the reward where the action  follows policy $\pi$, and $\Pi$ is a pre-specified policy class. More specifically, the joint probability of $(X, S, A, R)$ associated with $\E^\pi$ is $p( s,x) \pi(a| s,x) p(r | s,x, a)$, where $\pi(a |s,x)$ refers to the probability of choosing the action $a$ given the information $(s,x)$.

A widely used method for estimating $\pi^{\mathrm{opt}}$ involves two steps: policy evaluation and value maximization. In the policy evaluation step, an estimator $\widehat V(\pi)$ is obtained for the value function using  pre-collected historical data $\{S_i,X_i,A_i, R_i\}_{1 \leq i \leq n}$, where $n$ is the sample size. One can estimate the value function via the following method:\begin{gather*}
    \widehat V(\pi)=\frac{1}{n} \sum_{i=1}^n\sum_{a\in \mathcal{A}}\widehat r(S_i,X_i,a) \pi(a|S_i,X_i),
\end{gather*} where  $\widehat r(s,x,a)$ is the regression function estimator of the reward function $r(s,x,a)$. 
Such a  regression approach has been widely used in the policy learning literature; see, e.g., \citet{qian2011performance,moodie2012q}. Then in the value maximization step, the estimated optimal policy is derived by maximizing this estimated value function, within the pre-specified policy class $\Pi$:
\vspace{-0.6cm}
\begin{gather*}
    \widehat \pi \in \argmax_{\pi \in \Pi} \widehat V(\pi).
\end{gather*} 


\vspace{-0.8cm}
\subsection{Fairness Notions} \label{sec: fairness notion}

Recently, many fairness notions have been proposed under supervised learning setting, with particular attention to equal opportunity \citep{hardt2016equality} and counterfactual fairness \citep{kusner2017counterfactual}. Here, we extend these two notions to the policy learning framework and formally define both action fairness and outcome fairness under each concept.

\begin{defi} [Equal Opportunity]
    A policy $\pi$ is action-equal opportunity fair if for any $a$, $x$, and any $s\neq s'$,
    \vspace{-0.6cm} 
     \begin{gather*}
    \pi(a|s,x)=\pi(a|s',x). 
      \end{gather*} 
    A policy $\pi$ is outcome-equal opportunity fair if for any $a$, $x$, and any $s\neq s'$,
    \vspace{-0.6cm}
    \begin{gather*}
       \E^\pi \left(R^{(2)}\mid s,x\right)=\E^\pi \left(R^{(2)}\mid s',x\right).
        \label{eq:veo}
    \end{gather*} 
\end{defi}
\vspace{-0.6cm}
Equal opportunity requires that individuals with the same observed information $X$ be treated equally across different sensitive groups. This criterion is intuitively appealing and has been widely used in practice \citep{hardt2016equality}.
However, because it conditions only on observed covariates, it may fail to account for historical or unobserved sources of bias \citep{chen2023learning}. To address this, 
\citet{kusner2017counterfactual} proposed \emph{counterfactual fairness}, which adopts a causal perspective by asking whether the decision (or the induced outcome) would have remained unchanged had the individual belonged to a different sensitive group.

\begin{defi} [Counterfactual Fairness] \label{def:cf}
For an individual with observed attributes $(s, x)$, let $X_{s,x}(s')$ denote the counterfactual covariates had the individual's sensitive attribute been $s'$ instead of $s$.
A policy $\pi$ is action-counterfactually fair if $$\pi(a|s,X_{s,x}(s))=\pi(a|s',X_{s,x}(s')) \mbox{ for any } s' \neq s, \mbox{ almost surely.}$$ 
    and  attribute $x$.    
    A policy $\pi$ is outcome-counterfactually fair if for any $s\neq s'$, \begin{gather*}
\E^\pi \left(R^{(2)}|s,X_{s,x}(s)\right)=\E^\pi \left (R^{(2)}|s',X_{s,x}(s') \right)  \mbox{ almost surely.} 
    \end{gather*} 
\end{defi}
\vspace{-0.8cm}
Counterfactual fairness leverages counterfactual covariates to account for hidden biases that may not be captured by observed data. 
Action-counterfactual fairness implies the probability of choosing an action remains invariant under a counterfactual change to the sensitive attribute; and outcome-counterfactual fairness
implies that the expected benefit to the individual is independent of their sensitive attribute in a counterfactual world.


To conclude this section, we present the following examples to highlight the distinction between equal opportunity and counterfactual fairness. 

\begin{example}[Action Fairness]
{\it Equal opportunity.} Consider an educational admissions setting where gender is the sensitive attribute $S$ and test score is the observed covariate $X$. An admission decision is \emph{\it action--equal opportunity fair} if applicants with the same test score have the same probability of admission, regardless of gender. 

{\it Counterfactual fairness.} Suppose a female applicant, Nancy, has observed score $X_{s,x}(s)=59$. Now consider a counterfactual world in which Nancy were male; in that world, downstream effects of gender may change her test score to $X_{s,x}(s')=50$. 
An admission decision is \emph{action-counterfactually fair} if Nancy’s probability of admission in the real world (female, score 59) equals her probability of admission in this counterfactual world (male, score 50).
\end{example}

\begin{example} [Outcome fairness (continued)]
{\it Equal opportunity}. Suppose the value of an admission policy is defined as the probability of graduating with honors. The policy is \emph{outcome--equal opportunity fair} if applicants with the same test score have the same probability of graduating with honors, regardless of gender.

{\it Counterfactual fairness.} The policy is \emph{outcome-counterfactually fair} if Nancy’s real-world probability of graduating with honors (female, $X_{s,x}(s)=59$) equals her counterfactual probability (male, $X_{s,x}(s')=50$).
\end{example}

\subsection{Research Goal}

In fairness-aware policy learning, the decision maker’s objectives may be to (a) identify the policy with the highest value function among \textbf{\textit{action}} fairness policies; (b) find the policy with the highest value function among \textbf{\textit{outcome}} fairness policies; and (c) find the policy that satisfies both action and outcome fairness criteria  while maximizing the value function. 

Objective (a) is straightforward, as the policy maker can directly intervene in the action. For example, when considering equal opportunity fairness, one could simply using existing  algorithm but impose restriction on the policy class $\{\pi \in \Pi: \pi(a|x)\}$, thereby excluding the sensitive variable for decision making. However, objectives (b) and (c) are more challenging:  since direct intervention in the $R$--$(S,X,A)$ path is inaccessible,  one has to indirectly intervene in  $A$--$(S,X)$ to control outcome fairness. Subsequently, several natural questions arise:
(i) Does there exist a policy that ensures fairness in its corresponding outcome? And more ambitiously, can a policy achieve double fairness?
(ii) If no such fair policy exists, yet the decision maker still aims to mitigate both outcome and action unfairness, how should this be addressed?

In the following section, we explore these questions in depth. Specifically, we address Question (i) in Section~\ref{sec:exist} by establishing necessary and sufficient conditions for the existence of fairness policies. We examine Question (ii) in Section~\ref{sec: pareto}, where we adopt a multi-objective optimization framework to mitigate both value and action unfairness.

\section{Double Fairness Learning} \label{sec:double fair}




\subsection{Existence of Outcome Fairness Policy} \label{sec:exist}

Throughout this section, we focus on the equal opportunity fairness, and the discussion of  counterfactual fairness is deferred to Section \ref{app:cf} of the Appendix. We first present our main results on the existence of outcome fairness and double fairness policies. To begin with, we make the following two assumptions.

\begin{assumption} \label{asmp:domination}
    Given a subject with covariates $x$, there exists a treatment option $a^*(x)$ that dominates the other treatment with respect to the value function, regardless of the sensitive group, i.e., $\min_s f(s,x,a^*(x)) \geq \max_s f(s,x,1-a^*(x))$. 
\end{assumption}

\begin{assumption} \label{asmp:non-domination}
    (i) The effects of the sensitive variable have opposite signs across treatments, i.e., $[f(1,x,1)-f(0,x,1)][f(1,x,0)-f(0,x,0)] \leq 0$, for all $x$. \\
    (ii) The conditional treatment effects have opposite signs across sensitive variables, i.e., 
    $[f(1,x,1)-f(1,x,0)][f(0,x,1)-f(0,x,0)] \leq 0$,  for all $x$. \\
    (iii) For all $x$, there exists a pair  $(s^*(x),a^*(x))$ such that $f(s^*(x),x,a^*(x))\geq f(s^*(x),x,1-a^*(x))$, and $f(s^*(x),x,a^*(x))\geq f(1-s^*(x),x,a^*(x))$. 
\end{assumption}

Both Assumptions \ref{asmp:domination} and \ref{asmp:non-domination} require that the sensitive variable has only a moderate impact on the resulting value function. Specifically, Assumption \ref{asmp:domination} implies that the conditional treatment effects have the same sign across both sensitive attributes for a given $x$. As changes in the sensitive variable within the reward function do not alter the sign of the treatment effects, it suggests that the effects of the sensitive variable are relatively \textbf{\textit{mild}} compared to the treatment effects. As demonstrated in Proposition \ref{prop:vf}, if the effects of the sensitive variable are relatively limited, a outcome fairness policy exists, which is intuitively reasonable. If Assumption \ref{asmp:domination} does not hold (i.e., Assumption \ref{asmp:non-domination} (ii) applies), achieving outcome fairness requires additional regularity conditions on the effects of the sensitive variable. Specifically, Assumption \ref{asmp:non-domination} (i) requires that the effects of the sensitive variable have opposite signs across treatments for a given 
$x$, meaning that no sensitive group dominates the other under either treatment choice. Furthermore, when Assumption \ref{asmp:non-domination} (i) is satisfied, a unique double fairness policy exists (refer to Proposition \ref{prop:df}). We next summarize the main results.

\begin{prop} \label{prop:vf}
    There exists a outcome fairness policy for the given environment \textbf{if and only if} either Assumption \ref{asmp:domination} or Assumption \ref{asmp:non-domination} holds.
\end{prop}

\begin{prop} \label{prop:df}
A double fairness policy is either almost surely unique or does not exist. Furthermore, a  double fairness policy exists if and only if Assumption \ref{asmp:non-domination} (i) holds.
\end{prop}

Propositions \ref{prop:vf} and \ref{prop:df} establish the sufficient and necessary conditions  for the existence of the outcome fairness policy and the double fairness policy, respectively. A similar result holds for counterfactual fairness; see Section \ref{app:discrete} of the Supplemental Materials, which also discusses extensions of these results to multi-action settings beyond the binary case.

Since a outcome fairness or a double fairness policy may not always exist when performing the policy learning, one could instead seek a policy that is as fair as possible in both value function and action. In the next section, we present a multi-objective policy learning approach to achieve this goal. Notably, if a doubly fairness policy exists, our proposed method below can automatically identify it.

\subsection{Multi-objective Policy Learning} \label{sec: pareto}

Since achieving exact action and outcome fairness is generally challenging, we first introduce fairness metrics to quantify the degree of unfairness. Specifically, let $\Delta_1(\pi)$ and $\Delta_2(\pi)$ denote generic metrics for action fairness and outcome fairness, respectively, where smaller values indicate less unfairness.
Below, we present several examples of fairness metrics using \textit{equal opportunity} and \textit{counterfactual fairness}. For brevity, in what follows, we write $\pi(s, x)$ to represent $\pi(1 | s, x)$.

\begin{example}[Equal Opportunity]
    For equal opportunity fairness, a fairness metric for either action or value can be defined as following: \begin{align*}
    &\Delta_1(\pi)= \E\left[\pi(1,X)-\pi(0,X)\right]^2 \mbox{ and } \Delta_2(\pi)= \E\left[f^\pi (1,X)-f^\pi (0,X)\right]^2, 
\end{align*} where $f^\pi (s,x)=\E^\pi \left(R^{(2)}\mid s,x\right)$.  They can be estimated using \begin{align*}
    &\widehat \Delta_1(\pi)=\frac{1}{n}\sum_{i=1}^n \left[\pi(1,X_i)-\pi(0,X_i)\right]^2 \mbox{ and } \widehat\Delta_2(\pi)= \frac{1}{n}\sum_{i=1}^n\left[ \widehat f^\pi (1,X_i)-\widehat f ^\pi (0,X_i)\right]^2,
\end{align*} where $\widehat f^\pi (s,x)$ is some estimator for $f^\pi (s,x)$.
\end{example}

\begin{example}[Counterfactual Fairness]
    A counterfactual fairness metric for action or value can be defined as: \begin{align*}
    \Delta_1(\pi)= \E\left[\pi(S,X)-\pi(S',X')\right]^2 \mbox{ and } \Delta_2(\pi)= \E\left[f^\pi (S,X)-f^\pi (S',X')\right]^2, 
\end{align*} where $S'=1-S$, and $X'\equiv X_{S,X}(S')$. They can be estimated by \begin{gather*}
    \widehat \Delta_1(\pi)=\frac{1}{n}\sum_{i=1}^n \left[\pi(S_i,X_i)-\pi(S_i',\widehat X_i')\right]^2, 
    \mbox{ and } \widehat\Delta_2(\pi)= \frac{1}{n}\sum_{i=1}^n\left[ \widehat f^\pi (S_i,X_i)-\widehat f^\pi (S_i',\widehat X_i')\right]^2.  
\end{gather*} Note that here the counterfactual covariates $X_i'$  need to be estimated by $\widehat X_i'$. The entire estimation procedure will be discussed in Section~\ref{sec:app of eo and cf}.
\end{example}

Making optimal decisions in the presence of trade-offs among multiple objectives is inherently challenging, as it is generally impossible to optimize all objectives simultaneously. Therefore, it is crucial to understand how to effectively balance multiple objectives and make optimal trade-offs among them. The concept of Pareto set \citep[see, e.g.,][]{miettinen1999nonlinear} offers a valuable framework for this purpose, capturing the idea that improvements in one objective cannot be achieved without compromising at least one other. By leveraging the idea of Pareto set, we can navigate complex decision-making scenarios and strive for policies that optimize overall benefits across diverse objectives. 

We now incorporate the two previously defined fairness metrics to formulate the problem as a multi-objective policy learning task. Specifically, our proposal involves identifying the optimal Pareto fairness  policy within the Pareto  fairness  set \(\Pi_p\) that maximizes the value function:
\begin{gather} \label{eq:pi star}
    \pi^* \in \argmax_{\pi \in  \Pi_p}  V(\pi).
\end{gather} Here, \(\Pi_p\) denotes the set of Pareto fairness policies--those for which no other policy achieves strictly better action fairness and outcome fairness simultaneously. We formally define the Pareto fairness set for our problem below.

\begin{defi}[Pareto Fairness Policies and Pareto Fairness Set] \label{def:Pareto}
Let \(\Delta_1(\pi)\) and \(\Delta_2(\pi)\) denote generic fairness metrics for measuring the degree of action and outcome fairness, respectively. A policy \(\pi_p\) is said to be a \textbf{Pareto fairness policy} if there does not exist another policy \(\pi'\) such that \(\Delta_1(\pi') \leq \Delta_1(\pi_p)\) and \(\Delta_2(\pi') \leq \Delta_2(\pi_p)\), with at least one of these two inequalities being strict. The \textbf{Pareto fairness set}, denoted by \(\Pi_p\), is the collection of all such Pareto policies.
\end{defi}

\begin{remark}
Recall from Proposition \ref{prop:df} that if a double fairness policy $\bar\pi$ exists (i.e., $\Delta_1(\bar\pi) = \Delta_2(\bar\pi) = 0$), then it is almost surely unique. Since the fairness metrics are non-negative, no other policy $\pi'$ satisfies 
$\Delta_1(\pi') < \Delta_1(\bar\pi)$ or $\Delta_2(\pi') < \Delta_2(\bar\pi)$. Thus, by Definition \ref{def:Pareto}, $\Pi_p$ must be a singleton containing only $\bar\pi$. Consequently, if a double fairness policy exists, our proposed method in Equation~\eqref{eq:pi star} will automatically identify it.

\end{remark}



Next we demonstrate how to estimate $\pi^*$ as defined in Equation \eqref{eq:pi star}. Naturally, the estimation procedure consists of first approximating the Pareto set $\Pi_p$ by an estimated Pareto set $\widehat \Pi_p$, followed by choosing the policy within $\widehat\Pi_p$ that maximizes the estimated value function. Although Definition \ref{def:Pareto} offers an intuitive understanding of the Pareto fairness set, directly estimating the fairness Pareto set according to this definition can be challenging. To better characterize the Pareto fairness  set, we present the following useful proposition.

\begin{prop}[Lexicographic Weighted Tchebychef Method \citep{steuer1983interactive}]\label{prop: Cheby}
    The Pareto fairness  set can be computed via \begin{gather*} 
        \Pi_p= \bigcup_{\alpha\in(0,1)} \big\{\pi: \arginf_{\pi \in \Lambda_{\alpha}^*} \Delta(\pi) \big\}, \mbox{ where } \Delta(\pi)=\Delta_1(\pi)+\Delta_2(\pi),  \\
        \Lambda_{\alpha}^*=\big\{\pi:\arginf_{\pi \in \Pi} M_\alpha (\pi)\big\}, \mbox{ and }
          M_\alpha (\pi)=\max\left\{\alpha\Delta_1(\pi),(1-\alpha)\Delta_2(\pi)\right\},
    \end{gather*} where $\Pi$ is a pre-specified policy class.
\end{prop}

\begin{remark}
    \citet{viviano2023fair} adopted the linear scalarization approach (specific formulation discussed in Section \ref{sec:sims}) to characterize the Pareto set. However, this method  cannot identify Pareto fairness  policies in the non-convex regions of the Pareto fairness  front (as illustrated in Figure \ref{fig:ls_vs_tch} in Section \ref{app:linear vs tch} of the Supplemental Materials).  
    By contrast, the lexicographic weighted Tchebycheff approach 
    can identify the full set of Pareto fairness  policies without requiring convexity. We also include a simple intuitive example in Section \ref{app:linear vs tch} to demonstrate that the linear scalarization approach fails, while the weighted Tchebycheff method works.
\end{remark}

Now it remains to estimate the Pareto fairness  set via Proposition \ref{prop: Cheby}. The estimation procedure consists of the following steps. First, the fairness metrics $\Delta_1(\pi)$ and $\Delta_2(\pi)$ are estimated by $\widehat \Delta_1(\pi)$ and $\widehat \Delta_2(\pi)$, respectively. Then for any fixed $\alpha \in (0,1)$, the set $\Lambda_{\alpha}^*$ is estimated by \begin{gather} \label{eq:solution constraint}
\widehat\Lambda_{\alpha}^*=\left\{\pi\in \Pi:
\widehat M_{\alpha}(\pi) \leq \widehat M_{\alpha}^*+\kappa
\right\}, \\ \nonumber
\mbox{ where } \widehat M_{\alpha}(\pi)=\max\left\{\alpha\widehat\Delta_1(\pi),(1-\alpha)\widehat \Delta_2(\pi)\right\}, \mbox{ and } \widehat M_{\alpha}^*=\inf_{\pi\in \Pi} \widehat M_{\alpha}(\pi).
\end{gather} Here, $\kappa$ is a slack parameter, chosen to be a sufficiently small constant that ensures $\Lambda_{\alpha}^* \subseteq \widehat{\Lambda}_{\alpha}^*$ with high probability. 
The Pareto fairness set then can be estimated by \begin{align*} 
    \widetilde \Pi_p= \bigcup_{\alpha\in(0,1)} \left\{\pi \in \widehat\Lambda_{\alpha}^*: \widehat\Delta(\pi) \leq \widehat \Delta_{\alpha}^*+\kappa \right\},
\end{align*} 
where $\widehat\Delta(\pi)=\widehat\Delta_1(\pi)+\widehat\Delta_2(\pi)$, and $\widehat \Delta_{\alpha}^*= \inf_{\pi \in \widehat\Lambda_{\alpha}^*  }\widehat\Delta(\pi) $. Similarly, we introduce a slack parameter to ensure that the approximated Pareto fairness  set $\widetilde{\Pi}_p$ covers the true set $\Pi_p$ with high probability. Note that we can use the same slack parameter $\kappa$ as before, since the convergence rates of $\widehat{M}_{\alpha}(\pi)$ and $\widehat{\Delta}(\pi)$ are of the same order (as shown in Lemma \ref{lem:discretize error} in the supplemental material).


Since deriving the approximated set $\widetilde \Pi_p$ requires searching over all $\alpha \in (0,1)$, we discretize the interval $(0,1)$ by defining $\Omega = \{\alpha_1, \alpha_2, \dots, \alpha_K\}$, where $K$ denotes a pre-determined discretization level. Moreover, we require that for any $\alpha \in (0,1)$, there exists $\alpha_k \in \Omega$ such that $|\alpha - \alpha_k| \leq K^{-1}$.
We show later in Section \ref{sec:theory} that it suffices to set $K=O(\sqrt{n/ \log n})$ to guarantee desirable theoretical properties. Then we have \begin{align*} 
        \widehat \Pi_p= \bigcup_{k=1}^K \left\{\pi \in \widehat\Lambda_{\alpha_k}^*: \widehat\Delta(\pi) \leq \widehat \Delta_{\alpha_k}^*+\kappa\right\}.
    \end{align*}
  Finally, the estimated policy is \begin{gather} \label{eq:policy estimator}
        \widehat \pi \in  \argmax_{\pi \in \widehat \Pi_p} \widehat V(\pi).
    \end{gather}  We summarize the estimation procedure in Algorithm  \ref{alg}.

\begin{algorithm}[t]
	\caption{Pseudocodes for Estimating 
  $\widehat\pi$.}\label{alg}
	\begin{algorithmic}[1]
	
		\Function{}{$ \{S_i,X_i,A_i,R_i  \}_{1\leq i \leq n},\; \Pi, \kappa, K$}

   \State Calculate $\widehat\Delta_1(\pi)$, $\widehat\Delta_2(\pi)$, and $\widehat V(\pi)$, for $\pi \in \Pi$.
    \For {$k=1, 2 \dots , K$}

        \State Compute $\widehat M_{\alpha_k}^*=\inf_{\pi\in \Pi} \widehat M_{\alpha_k}(\pi)= \inf_{\pi\in \Pi}\max\{\alpha_k\widehat \Delta_1(\pi),(1-\alpha_k)\widehat\Delta_2(\pi)\}$.

    \State Compute $\widehat \Delta_{\alpha_k}^*= \inf_{\pi\in \Pi}\widehat\Delta(\pi) $, s.t. $\widehat M_{\alpha_k}(\pi) \leq \widehat M_{\alpha_k}^*+\kappa$.

\State Compute $\widehat V_k^* = \sup_{\pi\in \Pi} \widehat V(\pi)$, s.t. $\widehat \Delta(\pi) \leq \widehat \Delta_{\alpha_k}^*(\pi)+\kappa, \mbox{ and }\widehat M_{\alpha_k}(\pi) \leq \widehat M_{\alpha_k}^*+\kappa$. \par \,
Denote the solution as $\widehat \pi_k^*$.
\EndFor 
\State Compute $k^*= \argmax_{1 \leq k \leq K} \widehat V_k^* $ .

\State \Return $\widehat\pi= \widehat \pi_{k^*}^*$.
		\EndFunction
\end{algorithmic}

\end{algorithm}

\section{Theoretical Results} \label{sec:theory}

In this section, we first establish the theoretical properties of our method under general fairness notion and metric. Theorem \ref{them:regret} provides the regret bound for the estimated policy in terms of the Hausdorff distance  between the estimated and true Pareto sets; while Theorem~\ref{thm:pset} further bounds this distance. Both results hold for any fairness metric within our framework. We then specialize these results to equal opportunity and counterfactual fairness, deriving the corresponding regret bounds in Theorems~\ref{thm:eoapp} and \ref{thm:cfapp}.

We begin by defining the regret function.
The regret of $\widehat{\pi}$ quantifies the performance gap between  $\pi^*$ and $\widehat{\pi}$, and is defined as
\begin{gather} \label{eq:regret}
    \mu(\widehat{\pi}) \equiv \sup_{\pi \in \Pi_p} V(\pi) - V(\widehat{\pi}) = V(\pi^*) - V(\widehat{\pi}).
\end{gather} Unlike the standard regret function, which is always non-negative, our regret in Equation \eqref{eq:regret} can be negative because $\widehat{\pi}$ and $\pi^*$ are optimized over different sets, $\widehat{\Pi}_p$ and $\Pi_p$. Thus, establishing a lower bound for the regret is essential to demonstrate that the estimated policy closely mimics the optimal policy $\pi^*$ in terms of the resulting value function. Consequently, we derive a two-sided bound for the regret $\mu(\widehat{\pi})$ in the following theorem. Our analysis begins with the following bounded reward assumption, which is common in the policy learning literature, see, e.g., \citet{zhao2012estimating}.

\begin{assumption}\label{asmp:bounded reward}
    The functions $r(s,x,a)$, $\widehat{r}(s,x,a)$, $f(s,x,a)$, and $\widehat{f}(s,x,a)$ are uniformly bounded in absolute value by $R_{\max}$.
\end{assumption}

\begin{theorem}[Regret Bound] \label{them:regret}
Suppose Assumption \ref{asmp:bounded reward} holds, and the slack parameter $\kappa$ satisfies \vspace{-0.2cm}
\begin{gather*}
    \kappa \geq C \max\Big[\sup_{\pi \in \Pi} \big| \Delta_1(\pi) - \widehat{\Delta}_1(\pi)  \big|+\sup_{\pi \in \Pi} \big| \Delta_2(\pi) - \widehat{\Delta}_2(\pi)  \big|, \; K^{-1} \Big], 
\end{gather*}
for some constant $C$. Then the regret $\mu(\widehat{\pi})$ satisfies \vspace{-0.2cm}
\begin{gather*}
    -H_{\mathrm{TV}}(\Pi_p,\widehat{\Pi}_p) 
    \;\lesssim\; \mu(\widehat{\pi}) 
    \;\lesssim \; \min\Big[  \sup_{\pi \in \Pi_p} | \eta(\pi)|+ H_{\mathrm{TV}}(\Pi_p,\widehat{\Pi}_p), \; \sup_{\pi \in \Pi} 
    |\eta(\pi)| \Big] ,
\end{gather*} almost surely.
Here, $\eta(\pi) \equiv \widehat{V}(\pi) - V(\pi)$ denotes the policy evaluation error,  and 
$H_{\mathrm{TV}}(\Pi_p,\widehat{\Pi}_p)$ is the Hausdorff distance between $\Pi_p$ and $\widehat{\Pi}_p$ under the total variation metric\footnote{Here we use the total variation distance \(d_{\mathrm{TV}}\), so that
\begin{gather*}
H_{\mathrm{TV}}(\Pi_p,\widehat{\Pi}_p)
=
\max\Big\{
\sup_{\pi\in \Pi_p}\inf_{\hat\pi\in \widehat{\Pi}_p} d_{\mathrm{TV}}(\pi,\hat\pi),\;
\sup_{\hat\pi\in \widehat{\Pi}_p}\inf_{\pi\in \Pi_p} d_{\mathrm{TV}}(\hat\pi,\pi)
\Big\}, \\
d_{\mathrm{TV}}(\pi,\hat\pi)
:= \sup_{s,x} d_{\mathrm{TV}}\big(\pi(\cdot\mid s,x),\hat\pi(\cdot\mid s,x)\big)
= \sup_{s,x} \frac12\sum_{a\in\mathcal{A}} \big|\pi(a\mid s,x)-\hat\pi(a\mid s,x)\big|.
\end{gather*} }.
\end{theorem}

Theorem~\ref{them:regret} shows the regret consists of two parts:
(i) the uniform policy evaluation error over $\Pi_p$ and $\Pi$, and 
(ii) the estimation error of the Pareto fairness set, measured by the Hausdorff distance $H_{\mathrm{TV}}(\Pi_p,\widehat{\Pi}_p)$.
For (i), if the nuisance functions are estimated 
with sufficient accuracy and $\Pi$ is a VC subgraph class, then the uniform error admits a suitable bound; see, e.g., \citet{kitagawa2018should,athey2021policy}. 
We now analyze component (ii), starting with the following assumption.

\begin{assumption}[Well-separated Pareto fairness Set] \label{asmp: sep}
    Let $\Pi_\delta \equiv \left\{\pi \in \Pi: d_\mathrm{TV}(\pi,\Pi_p) > \delta \right \}$  denote the set of policies whose total variation distance from the true Pareto set  is larger than $\delta$, where $\delta$ is a positive constant. Assume that for all $  \alpha \in (0,1)$, 
    at least one of  the following conditions holds: \begin{gather*}
        (i) \inf_{\pi \in \Pi_\delta} \left|\Delta ( \pi) -\Delta_{\alpha}^* \right| > c \delta ; \quad (ii)\inf_{\pi \in \Pi_\delta} \left|M_{\alpha} ( \pi) -M_{\alpha}^* \right| > c \delta,
    \end{gather*} for some universal constant $c>0$.
\end{assumption}

Assumption~\ref{asmp: sep} imposes a \emph{separation} condition: policies that are at least \(\delta\) away from \(\Pi_p\) in total variation must incur an objective gap of order \(\delta\) under either the  criterion \(\Delta(\cdot)\) or the weighted Tchebycheff criterion \(M_\alpha(\cdot)\).
Such well-separated solution conditions are standard in establishing consistency of \(M\)- and \(Z\)-estimators; see, e.g., \citet{van2000asymptotic}.

\begin{theorem}\label{thm:pset}
     Suppose that the conditions in Theorem~\ref{them:regret} and Assumption~\ref{asmp: sep} hold, then \begin{gather*}
    H_{\mathrm{TV}}(\Pi_p,\widehat \Pi_p) \lesssim  \kappa,
\end{gather*} almost surely,
where $\kappa$ is the slack parameter in Algorithm \ref{alg}.
\end{theorem} \vspace{-0.8cm}

Theorem~\ref{thm:pset} shows that, up to a constant factor, the Hausdorff distance between the estimated and true Pareto fairness  sets is bounded by \(\kappa\). Recall that \(\kappa\) is determined by two terms: the estimation error in the fairness metrics and the discretization error \(K^{-1}\). If \(K^{-1}\) is chosen to be asymptotically negligible relative to the estimation error, then the resulting bound is driven by the estimation error.

Building on the general results above, the following subsection presents an in-depth analysis of equal opportunity and counterfactual fairness.

\subsection{Equal Opportunity and Counterfactual Fairness Applications} \label{sec:app of eo and cf}

In this subsection, we demonstrate how our theoretical results apply to the notions of equal opportunity and counterfactual fairness.
We begin by imposing some technical conditions.

\begin{assumption}[Finite VC Subgraph Dimension] \label{asmp: vc}
The policy class $\Pi$ has finite VC-subgraph dimension; that is, $\mbox{VC}(\Pi)$ is bounded.
\end{assumption}

Assumption~\ref{asmp: vc} is commonly imposed in policy learning literature, see, e.g., \citet{kitagawa2018should,athey2021policy}. 
Common examples of classes with finite VC-subgraph dimension include linear decision rules, decision trees, and sigmoid (logistic) score classes.


\begin{assumption}[Nuisance's Regularity Conditions] \label{asmp:nuisance}
There exists $ \gamma_1,\; \gamma_2 \in (0,1]$ such that the estimated reward function satisfies
    \begin{align*}
        &\max_{a,s} \; \frac{1}{n} \sum_{i=1}^n \left[\widehat r(s,X_i,a)-r(s,X_i,a)\right]^2=O_p(n^{-\gamma_1}),\\
        \mbox{ and }&\max_{a,s} \; \frac{1}{n} \sum_{i=1}^n \left[\widehat f(s,X_i,a)-f(s,X_i,a)\right]^2=O_p(n^{-\gamma_2}).
    \end{align*} 
\end{assumption}

Assumption \ref{asmp:nuisance} is needed because the outcome fairness metric estimator $\widehat \Delta_2(\pi)$ and the value estimator $\widehat V(\pi)$ are constructed from $\widehat r$ and $\widehat r$, respectively. The exponents $\gamma_1$ and $\gamma_2$ depend on the methods used to estimate $r$ and $f$. If $r$ and $f$ are estimated using correctly specified parametric models, one typically has $\gamma_1 = \gamma_2 = 1$. Alternatively,
flexible nonparametric methods such as sieve estimators \citep{huang1998projection} can be used. 
Under standard smoothness and complexity conditions, sieve estimators can achieve the optimal mean-squared rate $O_p(n^{-\frac{2\beta}{2\beta + d}})$, where $\beta$ is the smoothness parameter of the reward function, and $d$ is the dimension of $X$. 
Sieve methods are widely used in the policy evaluation and learning literature; see, e.g., \citet{shi2020statistical, chen2022well, bian2025off}.


For the counterfactual fairness notion, we 
additionally impose
the following Assumption \ref{asmp: additive error} to ensure that the counterfactual
covariates \(X_{s,x}(s')\) are identifiable from the observed data. Similar assumptions are used in \citet{chen2023learning} and
\citet{wang2025counterfactually} to study counterfactual fairness.

\begin{assumption} \label{asmp: additive error}
For each $j\in\{1,\ldots,d\}$, the covariate satisfies $X_{j}=\theta_j(S)+\varepsilon_{j}$, where $\theta_j(\cdot)$ is the regression function for the $j$th covariate, and the error term $\varepsilon_{j}$ satisfies $\E(\varepsilon_{j} |S)=0$.
\end{assumption} 

Under Assumption~\ref{asmp: additive error}, for an individual with
observed $(S,X)=(s,x)$, the counterfactual covariates under $S=s'$ are
\[
X_{s,x}(s')
=
\theta(s') + x - \theta(s), \mbox{ with }
\theta(\cdot)
=
\big(\theta_1(\cdot),\ldots,\theta_d(\cdot)\big)^\top .
\]

 We next impose a Lipschitz-type condition on the estimated reward function
and the policy class. This ensures that small errors in estimating the
counterfactual covariates lead to correspondingly small perturbations in
the quantities evaluated at $(s',x')$.

\begin{assumption}[Lipschitz Condition] \label{asmp:lip}

There exists a constant $m>0$ such that
\begin{align*}
   & (i)  \;\; |\widehat f(s,x,a)-\widehat f(s,x+t,a)| \leq m \norm{t}_2, \mbox{ for any } s, a, \mbox{ and } x; \\
 & (ii)  \sup_{\substack{
  \pi \in \Pi \\
 t: \, \norm{X - t}_2 \leq \delta
}} \E \left[\pi(S,X)- \pi(S,t)\right]^2 \leq m \delta. 
\end{align*}
\end{assumption}

Assumption \ref{asmp:lip}(i) is standard and has also been imposed by \citet{wang2025counterfactually} for counterfactual fairness policy learning. 
Assumption~\ref{asmp:lip}(ii) requires only Lipschitz continuity in expected
mean square, which is mild enough to accommodate deterministic policies
(for which pointwise Lipschitz continuity typically fails). A closely
related condition is used by \citet{shi2021concordance} when analyzing
linear deterministic decision rules; see Assumption~A5(ii) therein.

We now present the main regret bounds for equal opportunity and counterfactual fairness.

\begin{theorem}[Equal Opportunity] \label{thm:eoapp}
Suppose Assumptions \ref{asmp:bounded reward}, \ref{asmp: sep}, \ref{asmp: vc},  and \ref{asmp:nuisance} hold. In addition, set the slack parameter $\kappa$ to $c_0 \max(n^{-\frac{\gamma_2}{2}},\sqrt{\log n/n})$ for some constant $c_0>0$, and that $K^{-1} = O(\sqrt{\log n/n})$. Then under the \textbf{\textit{equal opportunity}} notion of fairness, 
\begin{gather*}
|\mu(\widehat{\pi})|=O_p\big(\sqrt{\mbox{VC}(\Pi)/n}+\sqrt{\log n/n}\big)+O_p\big(n^{-\frac{\gamma_1}{2}}+n^{-\frac{\gamma_2}{2}}\big).
   \end{gather*}
\end{theorem}
\vspace{-0.5cm}
\begin{theorem}[Counterfactual] \label{thm:cfapp}
Suppose the assumptions of Theorem \ref{thm:eoapp} hold. In addition, assume Assumptions \ref{asmp: additive error} and \ref{asmp:lip}. Then, under the \textbf{\textit{counterfactual}} notion of fairness, we have \begin{gather*}
|\mu(\widehat{\pi})|=O_p\big(\sqrt{\mbox{VC}(\Pi)/n}+\sqrt{\log n/n}+ \sqrt{d/n}\big)+O_p\big(n^{-\frac{\gamma_1}{2}}+n^{-\frac{\gamma_2}{2}}\big).
   \end{gather*}
\end{theorem}

Theorems \ref{thm:eoapp} and \ref{thm:cfapp} establish the regret bound of the estimated policy under the notions of equal opportunity fairness and counterfactual fairness, respectively. Specifically, the terms $O_p(\sqrt{\text{VC}(\Pi)/n})$ and $O_p(\sqrt{\log n/n})$ arise from the empirical process argument and concentration inequality, whereas the terms $O_p(n^{-\frac{\gamma_1}{2}})$ and $O_p(n^{-\frac{\gamma_2}{2}})$ reflects the estimation error in the outcome. 
 Relative to equal opportunity,
the counterfactual fairness bound includes an additional $\sqrt{d/n}$ term, stemming from the extra step of estimating counterfactual covariates.

\begin{remark}
In Theorems \ref{thm:eoapp} and \ref{thm:cfapp}, the slack parameter $\kappa$ involves the term $n^{-\frac{\gamma_2}{2}}$. In practice, if a parametric model is used, one can typically choose $\gamma_2 = 1$. For nonparametric regression methods such as the sieve estimator, $\gamma_2$ is given by $\frac{2\beta}{2\beta + d}$, as noted earlier. 
Since $\beta$  is
typically unknown, one may infer it from the optimal complexity choice.
For sieve estimation, the optimal number of basis functions satisfies
$
L^{\mathrm{opt}} \;=\; n^{\frac{d}{d+2\beta}} .
$
Thus, one can first select $L^{\text{opt}}$ (e.g., via cross-validation), and then
back out $\beta$ and hence $\gamma_2$ accordingly. The same logic
extends to other nonparametric methods (e.g., kernel or local polynomial
estimators), whose optimal bandwidth choices also depend on $\beta$; see, for example, \citet{Tsybakov2009}.
\end{remark}

\section{Numerical Studies} \label{sec:sims}

In this section, we demonstrate the effectiveness of our proposed framework through extensive simulation studies. 
We begin by outlining the data generating process in detail.

\textbf{\textit{Data generating procedure}}. We consider the following data-generating process, which is similar to that in Section F.2.1 of \citet{wang2025counterfactually}. The sensitive variable is given by $S_i \sim \mbox{Bernoulli}(0.35)$, and we consider the case in which the covariates are two-dimensional and are generated by $X_{1i}=0.45S_i+\varepsilon_{1i}$, and $X_{2i}=0.85S_i+\varepsilon_{2i}$, where $\{\varepsilon_{1i}\}_i$ and $\{\varepsilon_{2i}\}_i$ are independent error terms drawn from a standard Gaussian distribution. The action is sampled using $A_i \sim \mbox{Bernoulli}(\expit(-0.5+S_i+X_{1i}-X_{2i}))$, where  
$\expit(t)=\exp(t)/(1+\exp(t))$. Finally, we consider the case where $R=R^{(1)} = R^{(2)}$, such that $R$ is one-dimensional. In this setting, the reward is generated according to: $R_i=1.5S_i+0.9X_{1i}+0.8X_{2i}+A_i(1-0.8S_i-0.5X_{1i}-0.5X_{2i})+\epsilon_i, $  where $\{\epsilon_i\}_i$ are independent error terms drawn from a standard Gaussian distribution. Under this setting, the global optimal decision rule is a linear decision rule $\mathds 1(1-0.8S_i-0.5X_{1i}-0.5X_{2i}>0)$.

\textbf{\textit{Competing methods}}.
We compare our proposed method with the standard policy learning approach, and the approach for finding Pareto policies in \citet{viviano2023fair}. Specifically, the standard policy learning approach estimates the optimal policy as \begin{align*}
    \argmax_{\pi \in \Pi} \; \frac{1}{n} \sum_{i=1}^n \left[\widehat r(S_i,X_i,1)-\widehat r(S_i,X_i,0) \right]\pi(S_i,X_i).
\end{align*}  
As for the approach outlined in \citet{viviano2023fair}, it focuses solely on single fairness policy learning tasks, either outcome fairness or action fairness. To build on this, we first consider the original method based on the action fairness metric (VB1) and the outcome fairness metric (VB2). Moreover, we consider an adaptive version (ADVB), which has been tailored to accommodate the double fairness setting. Specifically, the ADVB method is adapted to our proposed double fairness setting  utilizing the linear scalarization method to derive the resulting policy. As in our proposed procedure, it involves two steps: first, the estimated Pareto set $\widehat \Pi_{ADVB}$ is obtained using linear scalarization, i.e., \begin{gather*}
    \left \{ \pi \in \Pi: \alpha_k\widehat\Delta_1(\pi)+(1-\alpha_k)\widehat\Delta_2(\pi) \leq \min_{\pi}\left(\alpha_k\widehat\Delta_1(\pi)+(1-\alpha_k)\widehat\Delta_2(\pi)\right) +\kappa, \forall \alpha_k \in \Omega \right \}.
\end{gather*} Subsequently, the estimated policy is obtained by maximizing the estimated value function within $\widehat \Pi_{ADVB}$:
$\argmax_{\pi \in \widehat \Pi_{ADVB}} \widehat V(\pi).$ 
Throughout this section, we apply a linear decision rule across all methods and evaluate their performance on a test set of 5,000 samples using both equal opportunity fairness and counterfactual fairness metrics. In addition, the value estimator is computed using a correctly specified outcome regression approach for all methods.

\begin{figure}[t] 
    \centering
    \includegraphics[width=\textwidth, height=6cm]{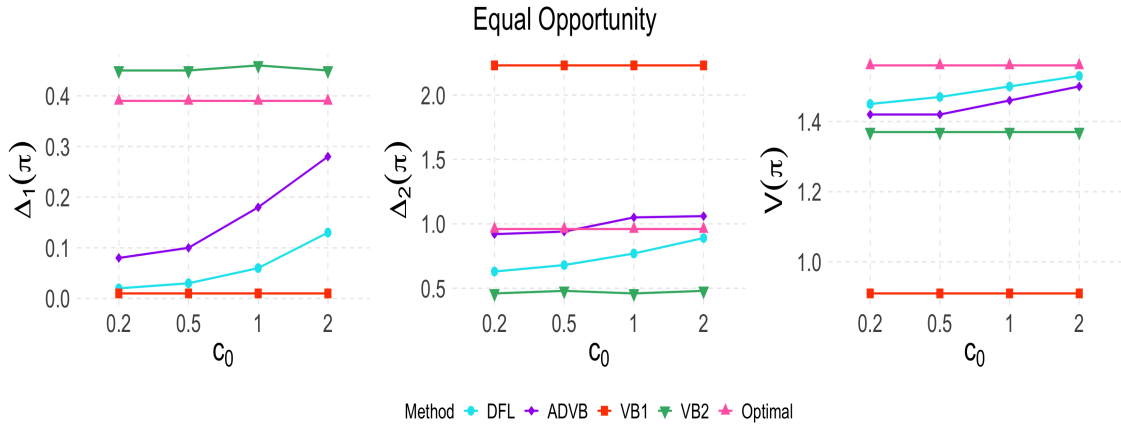}
    \caption{Empirical action fairness metric $\Delta_1(\pi)$ (lower is better), outcome fairness metric $\Delta_2(\pi)$ (lower is better), and value function $V(\pi)$ (higher is better), based on the \textbf{\textit{equal opportunity}} fairness notion. Results are obtained from a testing set of size 5,000, comparing our proposed DFL method with competing approaches. The results are averaged over 100 replications, each with $K=10$ and a training set of size 200.} 
    \label{fig:eo}
\end{figure}
In the first part of our experiments, we investigate the effect of the slack parameter $\kappa$. According to Theorems \ref{thm:eoapp} and \ref{thm:cfapp},  it is sufficient to set $\kappa = c_0 \sqrt{\log n / n}$ for some constant $c_0$, given that the reward function is estimated using a correctly specified model.
In the following, we consider four scenarios corresponding to different values of $c_0$: $0.2, 0.5, 1, $ and $2$, respectively. 

The simulation results for both equal opportunity and counterfactual fairness are summarized in Figures \ref{fig:eo} and \ref{fig:cf}, respectively. As evidenced by the results, our proposed method is not sensitive to the choice of $\kappa$. Moreover, across all scenarios, the proposed DFL method demonstrates superior performance. For instance, under the equal opportunity fairness notion, the DFL policy achieves a value function comparable to the unconstrained optimal policy while significantly enhancing fairness in both actions and value. Notably, our method achieves near-zero action fairness metrics for $c_0=0.2$; and its outcome fairness metrics are comparable to those achieved by the VB2 method, which focuses solely on minimizing outcome fairness. This underscores the DFL method's ability to effectively balance all three objectives—without compromising overall welfare or revenue.

\begin{figure}[t] 
    \centering
    \includegraphics[width=\textwidth, height=6cm]{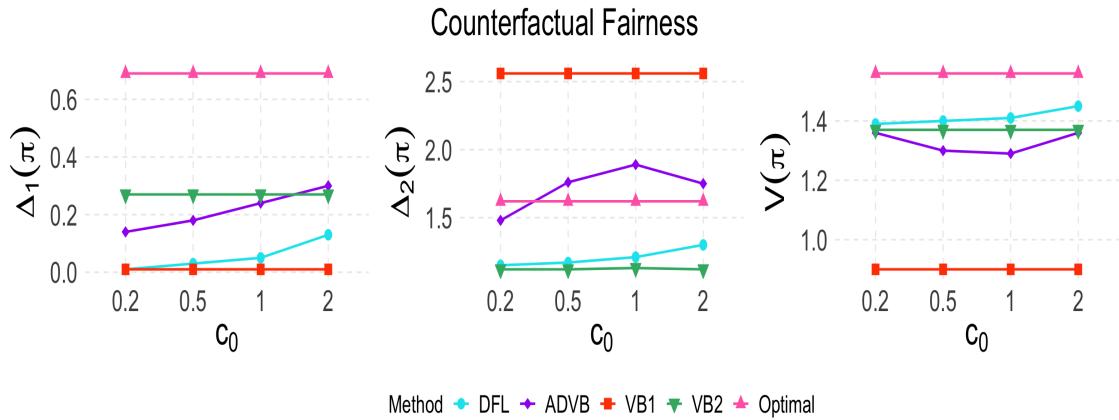}
    \caption{Empirical action fairness metric $\Delta_1(\pi)$ (lower is better), outcome fairness metric $\Delta_2(\pi)$ (lower is better),  and value  function $V(\pi)$ (higher is better) based on the \textbf{\textit{counterfactual}} fairness notion, obtained from a testing set of size 5,000, comparing our proposed DFL method with competing approaches. The results are averaged over 100 replications, each with $K=10$ and a training set of size 200.} 
    \label{fig:cf}
\end{figure}

\vspace{-.5cm}

\begin{figure}[t] 
    \centering
\includegraphics[width=\textwidth,height=5.5cm]{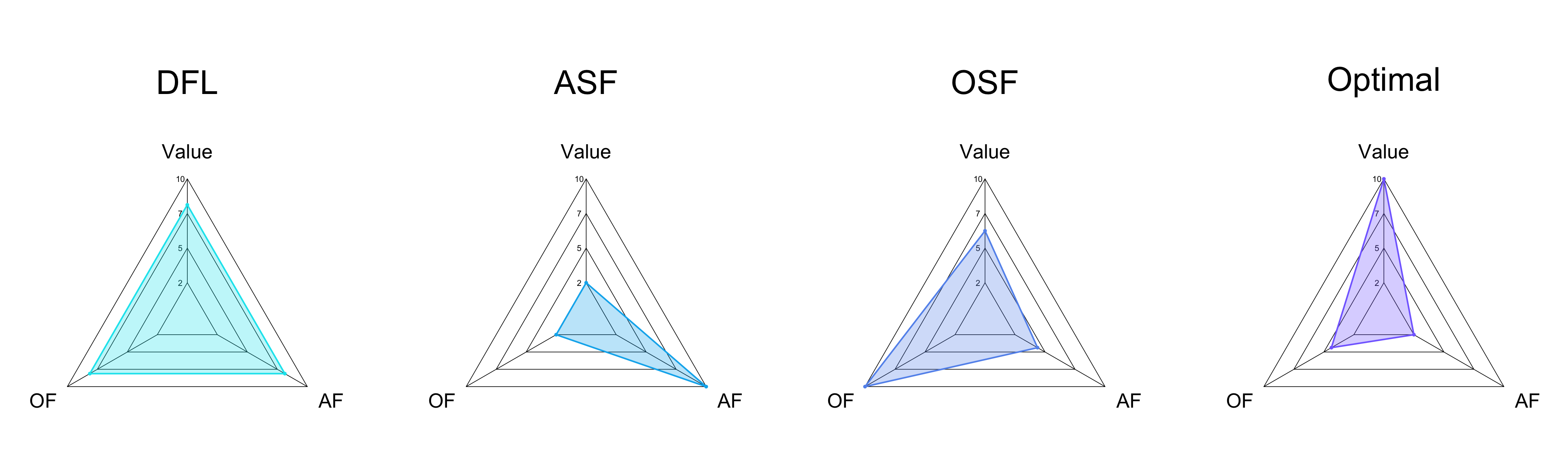}
    \caption{ Performance, based on numerical studies, is visualized using a radar chart. A higher score along a particular axis indicates better performance in that aspect. DFL refers to our proposed double fairness learning method. ASF represents the action single fair policy framework (VB1), and OSF denotes the outcome single fair policy framework (VB2). } 
    \label{fig:4radar}
\end{figure}

In contrast, the ADVB method results in a lower value function compared to DFL and exhibits higher action and outcome fairness metrics. Similarly, while the VB1 and VB2 methods prioritize single fairness criteria (action fairness for VB1 and outcome fairness for VB2), they perform poorly in terms of value function (both VB1 and VB2), outcome fairness (VB1), and action fairness (VB2). Overall, only the proposed DFL method achieves a flexible and effective balance among the three objectives, demonstrating both practical utility and adaptability. Figure \ref{fig:4radar} summarizes these results using a radar chart.

\begin{figure}[t] 
    \centering
    \includegraphics[width=\textwidth,height=8.5cm]{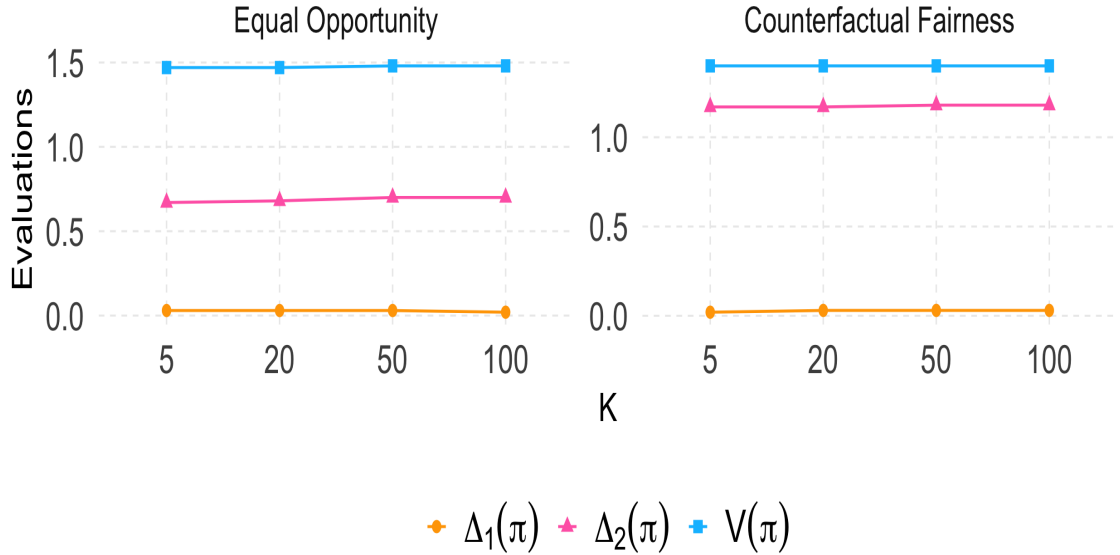}
    \caption{Empirical action fairness metric $\Delta_1(\pi)$ (lower is better), outcome fairness metric $\Delta_2(\pi)$ (lower is better),  and value  function $V(\pi)$ (higher is better), based on two fairness notions, evaluated from a testing set of size 5,000, using our proposed DFL method. The results are averaged over 100 replications, each with $\kappa=0.5 \sqrt{\log n /n}$ and a training set of size 200.} 
    \label{fig:K}
\end{figure}
In the previous results, the discretization size  $K$ was fixed at 10. In the second part of our experiments, we investigate the impact of different $K$ values on performance. For brevity, we only present our proposed DFL method, as the performance of all methods remains largely unaffected by changes in $K$. As shown in Figure \ref{fig:K}, varying the value of 
$K$ leads to negligible changes in performance.












\section{Real Data Analysis} \label{sec:real data}

\subsection{Motor Insurance Liability Claims Data} \label{sec:insurance}

In this section, we illustrate our proposed method to a Belgian motor third-party liability dataset from the CASdatasets \citep{dutang2019casdatasets}. 
Similar to \citet{xin2024antidiscrimination}, we use gender as the sensitive variable. In the offline dataset, we classify customers as low-risk or high-risk based on the bonus-malus (BM) scale, following common practice in the insurance literature, where drivers are often categorized as good or bad based on their BM score \citep{frangos2001design}. The BM scale ranges from 0 to 22, with higher values indicating greater risk. We define customers with BM greater than five as high-risk and the rest as low-risk. The dataset consists of 18,276 observations, with 28\% being female and 35\% classified as high-risk. The non-sensitive covariates include the number of claims, policyholder age, coverage plan, if the vehicle is part of a fleet, and the coverage duration. Since the dataset does not contain premium information, we generate the pure premium amount using the frequency-severity approach, as outlined by \citet{baumann2023fairness}. This approach sets pure premiums based on the predicted risk. To better reflect real-world pricing, we further adjust the final premium based on the coverage plan and gender. See Section \ref{app:data}  in Supplemental Materials for more details. Accordingly, the benefit or welfare received by the customer is the negative of the reward. Hence, it is essential to consider both action and outcome fairness in this setting.

\begin{table}[t] 
\caption{The estimated value function, the estimated action fairness metric, and the estimated outcome fairness metric using our proposed DFL and other competing methods.} 
\label{tab:ins}
\centering
\begin{tabular}{lrrrrrr} 
\toprule
& \multicolumn{3}{c}{Equal Opportunity} & \multicolumn{3}{c}{Counterfactual} \\
\cmidrule(l{4pt}r{4pt}){2-4} \cmidrule(l{4pt}r{4pt}){5-7} 
Method & $\widehat{V}(\widehat{\pi})$  & $\widehat{\Delta}_1(\widehat{\pi})$ & $\widehat{\Delta}_2(\widehat{\pi})$ & $\widehat{V}(\widehat{\pi})$  & $\widehat{\Delta}_1(\widehat{\pi})$ & $\widehat{\Delta}_2(\widehat{\pi})$ \\
\midrule
DFL & -10.16 &  0.03 &  0.18 & -10.18   &0.02  & 0.12 \\
Optimal & -9.38 & 0.33 & 1.14 & -9.38 & 0.38 & 1.10 \\
VB1 & -10.22  & 0.00 &  0.40 & -10.40 & 0.00 & 0.49 \\
VB2 &  -10.34  & 0.05  & 0.06 & -10.23 &  0.04 &  0.11 \\
ADVB & -10.17 &  0.00 &  0.48 & -10.40 &  0.00 &  0.49 \\
\bottomrule
\end{tabular}
\end{table}

We illustrate our method for both  equal opportunity fairness, and the counterfactual fairness. Throughout, we set $K = 10$ and $\kappa = 2 \sqrt{\log n / n}$. Table \ref{tab:ins} provides a summary of results. Similar to findings from previous numerical studies, our proposed policy yields a estimated value function similar to the estimated optimal policy, while being much fairer than the optimal policy in terms of both action fairness and outcome fairness. This demonstrates that the newly proposed approach effectively balances the three objectives without compromising the overall revenue.

\subsection{Entrepreneurship Training Data} \label{sec:entre}

In this subsection, we illustrate our proposed method to the entrepreneurship training data for North America undergraduate students from \citet{lyons2017impact}. This study examined whether entrepreneurship programs have varying impacts on certain minority groups, such as females or non-Caucasians. Focusing solely on the outcome fairness criterion, they found that the program's short-term effect on minorities is minimal. The goal in our study is to identify the  policy that maximizes the expected subsequent entrepreneurial activity for the population while keeping the action and value unfairness  as low as possible.

In this context, the action represents the decision of whether to accept or reject admission into the program; while the reward  is defined as subsequent entrepreneurial activity, specifically whether finalists work in the startup sector during the period immediately following the program. 
We focus on the sensitive variable gender. The covariates in this dataset include the average interview score, prior entrepreneurial activity, and school rank. The dataset includes a total of 335 subjects, with 53.43\% admitted and the remaining not admitted. Additionally, 26.57\% of the subjects are female; among those who were admitted, 27.93\% are female. 

Since the reward is binary, we fit the outcome model using logistic regression. Note that during the training process, the optimal policy for this dataset turns out to be a one-size-fits-all approach, where ideally, everyone would be admitted to the program. This makes the action exactly fair, placing the optimal policy within our Pareto efficiency set. To address this issue, we introduce an additional constraint on the pre-defined policy class, ensuring that the estimated number of admitted individuals does not exceed the number of admitted individuals in the dataset (179 in total). This constraint is practical, as many real-world applications face limitations in resources and budget.

\vspace{.4 cm}

\begin{table}[t] 
\caption{The estimated value function, the estimated action fairness metric, and the estimated outcome fairness metric using our proposed DFL and other competing methods.} 
\label{tab:real_data}
\centering
\begin{tabular}{lrrrrrr} 
\toprule
& \multicolumn{3}{c}{Equal Opportunity} & \multicolumn{3}{c}{Counterfactual} \\
\cmidrule(l{4pt}r{4pt}){2-4} \cmidrule(l{4pt}r{4pt}){5-7} 
Method & $\widehat{V}(\widehat{\pi})$  & $\widehat{\Delta}_1(\widehat{\pi})$ & $\widehat{\Delta}_2(\widehat{\pi})$ & $\widehat{V}(\widehat{\pi})$  & $\widehat{\Delta}_1(\widehat{\pi})$ & $\widehat{\Delta}_2(\widehat{\pi})$ \\
\midrule

DFL & 0.22 &0.09& 0.06 & 0.22 &0.09 &0.07 \\
Optimal & 0.23& 0.68& 0.12 & 0.23 &0.69& 0.13 \\
VB1 & 0.15& 0.00 &0.02 & 0.16 &0.00 &0.03 \\
VB2 &  0.16 &0.32& 0.02 & 0.18 &0.53 &0.02 \\
ADVB & 0.22 &0.26& 0.08 & 0.22 &0.10& 0.07 \\
\bottomrule
\end{tabular}
\end{table}

Table \ref{tab:real_data} provides a summary of the analysis results. As the results show, our proposed method produces an estimated value function similar to the estimated optimal policy, while being considerably fairer in terms of both action fairness and outcome fairness.

\section{Discussion} \label{sec:discussion}


This work  studies double fairness in policy learning by simultaneously addressing three objectives.
We provide a systematic analysis to identify conditions under which a outcome-fair or double-fair policy can exist. We also propose a novel estimation framework that incorporates fairness constraints directly into the multi-objective policy learning problem.
Our approach is flexible and accommodates a range of fairness notions. Extensive experiments on both synthetic and real-world datasets validate our method, showcasing its effectiveness in achieving double fairness without compromising overall welfare.




\bibliographystyle{apalike}
\bibliography{references}

@inproceedings{chen2022well,
  title={On well-posedness and minimax optimal rates of nonparametric q-function estimation in off-policy evaluation},
  author={Chen, Xiaohong and Qi, Zhengling},
  booktitle={International Conference on Machine Learning},
  pages={3558--3582},
  year={2022},
  organization={PMLR}
}

@article{qian2011performance,
  title={Performance guarantees for individualized treatment rules},
  author={Qian, Min and Murphy, Susan A},
  journal={Annals of statistics},
  volume={39},
  number={2},
  pages={1180},
  year={2011}
}

@article{moodie2012q,
  title={Q-learning for estimating optimal dynamic treatment rules from observational data},
  author={Moodie, Erica EM and Chakraborty, Bibhas and Kramer, Michael S},
  journal={Canadian Journal of Statistics},
  volume={40},
  number={4},
  pages={629--645},
  year={2012},
  publisher={Wiley Online Library}
}

@article{frangos2001design,
  title={Design of optimal bonus-malus systems with a frequency and a severity component on an individual basis in automobile insurance},
  author={Frangos, Nicholas E and Vrontos, Spyridon D},
  journal={ASTIN Bulletin: The Journal of the IAA},
  volume={31},
  number={1},
  pages={1--22},
  year={2001},
  publisher={Cambridge University Press}
}

@article{baumann2023fairness,
  title={Fairness and risk: an ethical argument for a group fairness definition insurers can use},
  author={Baumann, Joachim and Loi, Michele},
  journal={Philosophy \& Technology},
  volume={36},
  number={3},
  pages={45},
  year={2023},
  publisher={Springer}
}

@article{wang2025counterfactually,
  title={Counterfactually Fair Reinforcement Learning via Sequential Data Preprocessing},
  author={Wang, Jitao and Shi, Chengchun and Piette, John D and Loftus, Joshua R and Zeng, Donglin and Wu, Zhenke},
  journal={arXiv preprint arXiv:2501.06366},
  year={2025}
}

@article{talia2024algorithms,
  title={Algorithms That Can Deny Care, and a Call for AI Explainability},
  author={Talia, Domenico},
  journal={Computer},
  volume={57},
  number={7},
  pages={109--112},
  year={2024},
  publisher={IEEE}
}

@article{xin2024antidiscrimination,
  title={Antidiscrimination insurance pricing: Regulations, fairness criteria, and models},
  author={Xin, Xi and Huang, Fei},
  journal={North American Actuarial Journal},
  volume={28},
  number={2},
  pages={285--319},
  year={2024},
  publisher={Taylor \& Francis}
}

@article{dutang2019casdatasets,
  title={CASdatasets: insurance datasets},
  author={Dutang, Christophe and Charpentier, Arthur},
  journal={R package version},
  pages={1--0},
  year={2019}
}

@article{lyons2017impact,
  title={The impact of entrepreneurship programs on minorities},
  author={Lyons, Elizabeth and Zhang, Laurina},
  journal={American Economic Review},
  volume={107},
  number={5},
  pages={303--307},
  year={2017},
  publisher={American Economic Association 2014 Broadway, Suite 305, Nashville, TN 37203}
}

@inproceedings{schulman2015trust,
  title={Trust region policy optimization},
  author={Schulman, John and Levine, Sergey and Abbeel, Pieter and Jordan, Michael and Moritz, Philipp},
  booktitle={International conference on machine learning},
  pages={1889--1897},
  year={2015},
  organization={PMLR}
}

@article{viviano2023fair,
  title={Fair policy targeting},
  author={Viviano, Davide and Bradic, Jelena},
  journal={Journal of the American Statistical Association},
  pages={1--14},
  year={2023},
  publisher={Taylor \& Francis}
}

@inproceedings{grgic2016case,
  title={The case for process fairness in learning: Feature selection for fair decision making},
  author={Grgic-Hlaca, Nina and Zafar, Muhammad Bilal and Gummadi, Krishna P and Weller, Adrian},
  booktitle={NIPS symposium on machine learning and the law},
  volume={1},
  pages={11},
  year={2016},
  organization={Barcelona, Spain}
}

@article{hardt2016equality,
  title={Equality of opportunity in supervised learning},
  author={Hardt, Moritz and Price, Eric and Srebro, Nati},
  journal={Advances in neural information processing systems},
  volume={29},
  year={2016}
}

@inproceedings{nabi2019learning,
  title={Learning optimal fair policies},
  author={Nabi, Razieh and Malinsky, Daniel and Shpitser, Ilya},
  booktitle={International Conference on Machine Learning},
  pages={4674--4682},
  year={2019},
  organization={PMLR}
}

@article{tan2022rise,
  title={Rise: Robust individualized decision learning with sensitive variables},
  author={Tan, Xiaoqing and Qi, Zhengling and Seymour, Christopher and Tang, Lu},
  journal={Advances in Neural Information Processing Systems},
  volume={35},
  pages={19484--19498},
  year={2022}
}

@article{steuer1983interactive,
  title={An interactive weighted Tchebycheff procedure for multiple objective programming},
  author={Steuer, Ralph E and Choo, Eng-Ung},
  journal={Mathematical programming},
  volume={26},
  pages={326--344},
  year={1983},
  publisher={Springer}
}

@article{kitagawa2018should,
  title={Who should be treated? empirical welfare maximization methods for treatment choice},
  author={Kitagawa, Toru and Tetenov, Aleksey},
  journal={Econometrica},
  volume={86},
  number={2},
  pages={591--616},
  year={2018},
  publisher={Wiley Online Library}
}

@inproceedings{dwork2012fairness,
  title={Fairness through awareness},
  author={Dwork, Cynthia and Hardt, Moritz and Pitassi, Toniann and Reingold, Omer and Zemel, Richard},
  booktitle={Proceedings of the 3rd innovations in theoretical computer science conference},
  pages={214--226},
  year={2012}
}

@book{Tsybakov2009,
  author    = {Tsybakov, Alexandre B.},
  title     = {Introduction to Nonparametric Estimation},
  year      = {2009},
  series    = {Springer Series in Statistics},
  publisher = {Springer},
  address   = {New York, NY},
  doi       = {10.1007/b13794},
  isbn      = {978-0-387-79051-0}
}

@book{miettinen1999nonlinear,
  title={Nonlinear multiobjective optimization},
  author={Miettinen, Kaisa},
  year={1999},
  publisher={Springer Science \& Business Media}
}

@book{van2000asymptotic,
  title={Asymptotic statistics},
  author={Van der Vaart, Aad W},
  volume={3},
  year={2000},
  publisher={Cambridge university press}
}

@incollection{van1996weak,
  title={Weak convergence},
  author={Van Der Vaart, Aad W and Wellner, Jon A},
  booktitle={Weak convergence and empirical processes: with applications to statistics},
  pages={16--28},
  year={1996},
  publisher={Springer}
}

@article{kusner2017counterfactual,
  title={Counterfactual fairness},
  author={Kusner, Matt J and Loftus, Joshua and Russell, Chris and Silva, Ricardo},
  journal={Advances in neural information processing systems},
  volume={30},
  year={2017}
}

@article{chen2023learning,
  title={On learning and testing of counterfactual fairness through data preprocessing},
  author={Chen, Haoyu and Lu, Wenbin and Song, Rui and Ghosh, Pulak},
  journal={Journal of the American Statistical Association},
  pages={1--11},
  year={2023},
  publisher={Taylor \& Francis}
}

@article{cohen2022price,
  title={Price discrimination with fairness constraints},
  author={Cohen, Maxime C and Elmachtoub, Adam N and Lei, Xiao},
  journal={Management Science},
  volume={68},
  number={12},
  pages={8536--8552},
  year={2022},
  publisher={INFORMS}
}

@article{bian2026beyond,
  title={Beyond Demand Estimation: Consumer Surplus Evaluation via Cumulative Propensity Weights},
  author={Bian, Zeyu and Biggs, Max and Gao, Ruijiang and Qi, Zhengling},
  journal={arXiv preprint arXiv:2601.01029},
  year={2026}
}

@article{cui2025policy,
  title={Policy learning with distributional welfare},
  author={Cui, Yifan and Han, Sukjin},
  journal={Journal of the American Statistical Association},
  pages={1--12},
  year={2025},
  publisher={Taylor \& Francis}
}

@article{shi2021concordance,
  title     = {Concordance and value information criteria for optimal treatment decision},
  author    = {Shi, Chengchun and Song, Rui and Lu, Wenbin},
  journal   = {Annals of Statistics},
  volume    = {49},
  number    = {1},
  pages     = {49--75},
  year      = {2021},
  doi       = {10.1214/19-AOS1908}
}

@article{bian2023variable,
  title={Variable selection in regression-based estimation of dynamic treatment regimes},
  author={Bian, Zeyu and Moodie, Erica EM and Shortreed, Susan M and Bhatnagar, Sahir},
  journal={Biometrics},
  volume={79},
  number={2},
  pages={988--999},
  year={2023},
  publisher={Wiley Online Library}
}

@article{qi2020multi,
  title={Multi-armed angle-based direct learning for estimating optimal individualized treatment rules with various outcomes},
  author={Qi, Zhengling and Liu, Dacheng and Fu, Haoda and Liu, Yufeng},
  journal={Journal of the American Statistical Association},
  volume={115},
  number={530},
  pages={678--691},
  year={2020},
  publisher={Taylor \& Francis}
}

@article{wallace2015doubly,
  title={Doubly-robust dynamic treatment regimen estimation via weighted least squares},
  author={Wallace, Michael P and Moodie, Erica EM},
  journal={Biometrics},
  volume={71},
  number={3},
  pages={636--644},
  year={2015},
  publisher={Oxford University Press}
}

@article{javanmard2019dynamic,
  title={Dynamic pricing in high-dimensions},
  author={Javanmard, Adel and Nazerzadeh, Hamid},
  journal={Journal of Machine Learning Research},
  volume={20},
  number={9},
  pages={1--49},
  year={2019}
}

@article{den2015dynamic,
  title={Dynamic pricing and learning: historical origins, current research, and new directions},
  author={Den Boer, Arnoud V},
  journal={Surveys in operations research and management science},
  volume={20},
  number={1},
  pages={1--18},
  year={2015},
  publisher={Elsevier}
}

@article{bian2024tale,
  title={A Tale of Two Cities: Pessimism and Opportunism in Offline Dynamic Pricing},
  author={Bian, Zeyu and Qi, Zhengling and Shi, Cong and Wang, Lan},
  journal={arXiv preprint arXiv:2411.08126},
  year={2024}
}

@article{manski2004statistical,
  title={Statistical treatment rules for heterogeneous populations},
  author={Manski, Charles F},
  journal={Econometrica},
  volume={72},
  number={4},
  pages={1221--1246},
  year={2004},
  publisher={Wiley Online Library}
}

@article{zhao2012estimating,
  title={Estimating individualized treatment rules using outcome weighted learning},
  author={Zhao, Yingqi and Zeng, Donglin and Rush, A John and Kosorok, Michael R},
  journal={Journal of the American Statistical Association},
  volume={107},
  number={499},
  pages={1106--1118},
  year={2012},
  publisher={Taylor \& Francis}
}

@article{cui2021semiparametric,
  title={A semiparametric instrumental variable approach to optimal treatment regimes under endogeneity},
  author={Cui, Yifan and Tchetgen Tchetgen, Eric},
  journal={Journal of the American Statistical Association},
  volume={116},
  number={533},
  pages={162--173},
  year={2021},
  publisher={Taylor \& Francis}
}

@article{wang2018quantile,
  title={Quantile-optimal treatment regimes},
  author={Wang, Lan and Zhou, Yu and Song, Rui and Sherwood, Ben},
  journal={Journal of the American Statistical Association},
  volume={113},
  number={523},
  pages={1243--1254},
  year={2018},
  publisher={Taylor \& Francis}
}

@article{shi2018high,
  title={High-dimensional A-learning for optimal dynamic treatment regimes},
  author={Shi, Chengchun and Fan, Alin and Song, Rui and Lu, Wenbin},
  journal={Annals of statistics},
  volume={46},
  number={3},
  pages={925},
  year={2018},
  publisher={NIH Public Access}
}

@inproceedings{gest,
  title={Optimal structural nested models for optimal sequential decisions},
  author={Robins, James M},
  editor={Lin, Danyu Y and Heagerty, PJ},
  booktitle={Proceedings of the Second Seattle Symposium in Biostatistics},
  address   = {New York},
  pages={189--326},
  year={2004},
  organization={Springer}
}

@article{Murphy,
  title={Optimal dynamic treatment regimes},
  author={Murphy, Susan A},
  journal={Journal of the Royal Statistical Society: Series B},
  volume={65},
  number={2},
  pages={331--355},
  year={2003},
  publisher={Wiley Online Library}
}

@article{athey2021policy,
  title={Policy learning with observational data},
  author={Athey, Susan and Wager, Stefan},
  journal={Econometrica},
  volume={89},
  number={1},
  pages={133--161},
  year={2021},
  publisher={Wiley Online Library}
}

@inproceedings{wu2019counterfactual,
  title={Counterfactual fairness: Unidentification, bound and algorithm},
  author={Wu, Yongkai and Zhang, Lu and Wu, Xintao},
  booktitle={Proceedings of the twenty-eighth international joint conference on Artificial Intelligence},
  year={2019}
}

@inproceedings{zemel2013learning,
  title={Learning fair representations},
  author={Zemel, Rich and Wu, Yu and Swersky, Kevin and Pitassi, Toni and Dwork, Cynthia},
  booktitle={International conference on machine learning},
  pages={325--333},
  year={2013},
  organization={PMLR}
}

@article{lin2024smooth,
  title={Smooth Tchebycheff Scalarization for Multi-Objective Optimization},
  author={Lin, Xi and Zhang, Xiaoyuan and Yang, Zhiyuan and Liu, Fei and Wang, Zhenkun and Zhang, Qingfu},
  journal={arXiv preprint arXiv:2402.19078},
  year={2024}
}

@article{huang1998projection,
  title={Projection estimation in multiple regression with application to functional ANOVA models},
  author={Huang, Jianhua Z},
  journal={The Annals of statistics},
  volume={26},
  number={1},
  pages={242--272},
  year={1998},
  publisher={Institute of Mathematical Statistics}
}

@article{shi2020statistical,
  title={Statistical inference of the value function for reinforcement learning in infinite-horizon settings},
  author={Shi, Chengchun and Zhang, Sheng and Lu, Wenbin and Song, Rui},
  journal={Journal of the Royal Statistical Society Series B: Statistical Methodology},
  volume={84},
  number={3},
  pages={765--793},
  year={2022},
  publisher={Oxford University Press}
}

@article{fang2023fairness,
  title={Fairness-oriented learning for optimal individualized treatment rules},
  author={Fang, Ethan X and Wang, Zhaoran and Wang, Lan},
  journal={Journal of the American Statistical Association},
  volume={118},
  number={543},
  pages={1733--1746},
  year={2023},
  publisher={Taylor \& Francis}
}

@article{bian2025off,
  title={Off-policy evaluation in doubly inhomogeneous environments},
  author={Bian, Zeyu and Shi, Chengchun and Qi, Zhengling and Wang, Lan},
  journal={Journal of the American Statistical Association},
  volume={120},
  number={550},
  pages={1102--1114},
  year={2025},
  publisher={Taylor \& Francis}
}

\begin{appendices}

\newpage

\begin{center}
    {\LARGE\bf Supplemental Materials for ``Doubly Trustworthy Policy Learning''}
\end{center}

\section{Additional Results}

\subsection{Limitation of the Linear Scalarization Approach} \label{app:linear vs tch}

In this subsection, we show that the limitation of the linear scalarization approach via the following simple example. 

\begin{figure}[H] 
    \centering
    \begin{subfigure}[b]{0.49\textwidth}
        \centering
        \includegraphics[width=\textwidth]{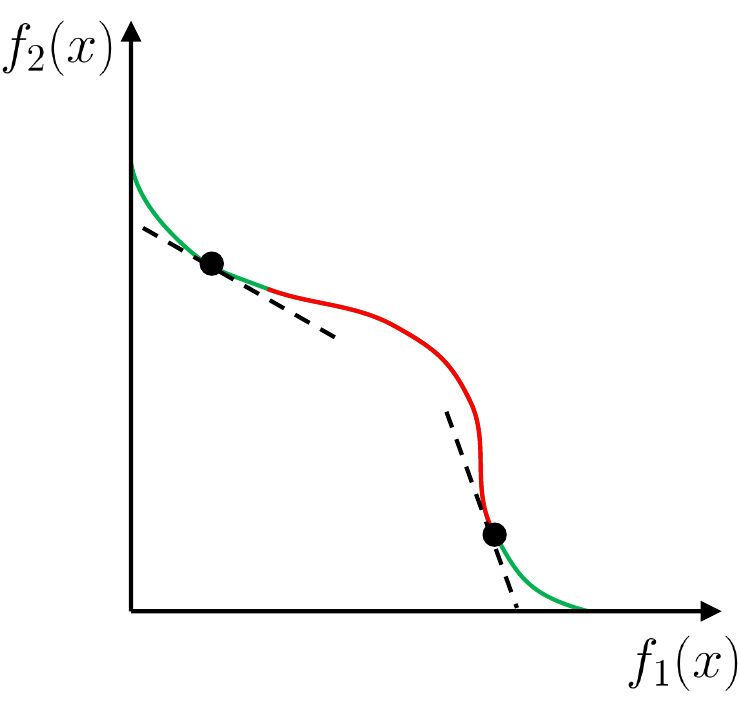}
        \caption{Linear scalarization.}
        \label{fig:ls}
    \end{subfigure}
    \hfill
    \begin{subfigure}[b]{0.49\textwidth}
        \centering
        \includegraphics[width=\textwidth]{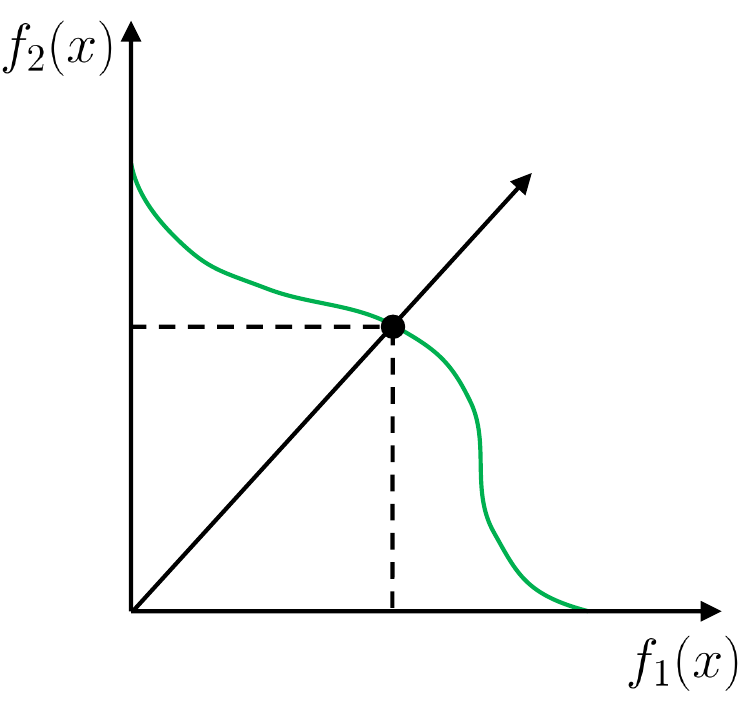}
        \caption{Weighted Tchebycheff scalarization.}
        \label{fig:tch}
    \end{subfigure}
    \caption{ Linear scalarization approach cannot identify Pareto polices in the non-convex regions of the Pareto front, as these solutions lack supporting hyperplanes \citep{miettinen1999nonlinear}. In contrast, weighted Tchebycheff scalarization can find all Pareto policies, regardless of the shape of the objective space. Figure adapted from \citet{lin2024smooth}.} 
    \label{fig:ls_vs_tch}
\end{figure}

Consider the following data generation procedure: $S\sim \mbox{Ber}(0.5)$, $X\sim \mbox{Ber}(0.6)$, $R=S+X+0.4 A- 0.5 AS$. The policy class $\Pi$ consists the following three policies:  $\pi_1=\mathds1 (X>0)$, $\pi_2=\mathds1 (S+X>0)$, $\pi_3=\mathds1 (S>0)$. That is, $\Pi=\{ \pi_1, \pi_2, \pi_3  \}$.







\begin{figure}[t] 
    \centering
\includegraphics[width=0.75\textwidth]{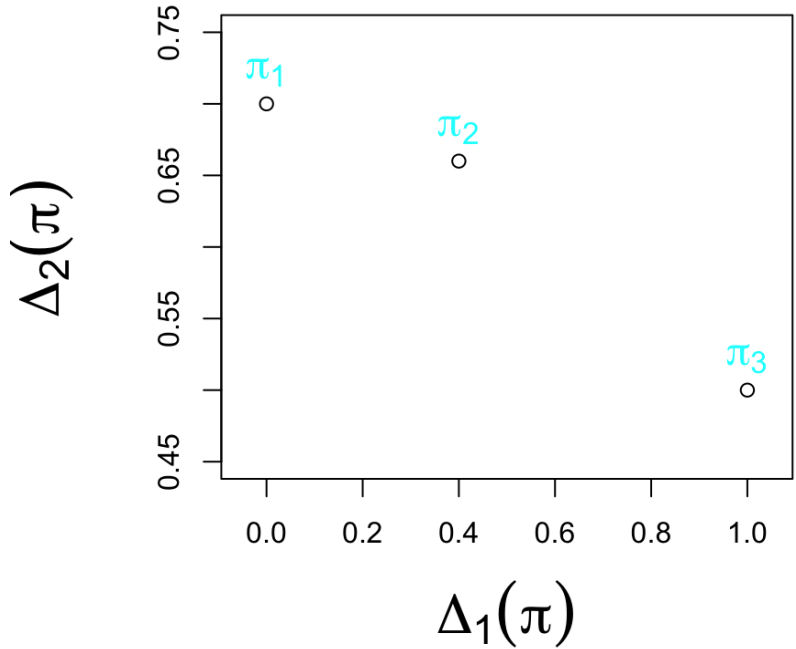}
    \caption{The Pareto front.} 
    \label{fig:example}
\end{figure}

Under this data generation process, it can be easily shown that \begin{gather*}
    \Delta_1(\pi_1)=0, \Delta_2(\pi_1)=0.7,\\
    \Delta_1(\pi_2)=0.4, \Delta_2(\pi_2)=0.66,\\
    \Delta_1(\pi_3)=1, \Delta_2(\pi_3)=0.5,
\end{gather*} for the equal opportunity fairness. See Figure \ref{fig:example} below for the Pareto front. Clearly, all the three polices lies in the Pareto  set, i.e., they cannot Pareto dominate each other.

Using the linear scalarization approach, it can be shown that \begin{gather*}
    \left \{ \pi \in \Pi: \argmin_{\pi} \alpha\Delta_1(\pi)+(1-\alpha)\Delta_2(\pi), \; \forall \alpha\in (0,1) \right \}= \{\pi_1, \pi_3\}.
\end{gather*} That is, there does not exist a $\alpha \in (0,1)$ such that $\alpha\Delta_1(\pi_2)+(1-\alpha)\Delta_2(\pi_2)<\alpha\Delta_1(\pi_1)+(1-\alpha)\Delta_2(\pi_1)$, and $\alpha\Delta_1(\pi_2)+(1-\alpha)\Delta_2(\pi_2)<\alpha\Delta_1(\pi_3)+(1-\alpha)\Delta_2(\pi_3)$. As a result, the linear scalarization can only identify $\pi_1$ and $\pi_3$, while leaving the policy that can yield the highest value-$\pi_2$ out of the Pareto set.

In comparison, the lexicographic weighted Tchebyshev scalarization can successfully identify all the three policies. To show this, only basic algebraic calculations are needed; hence, the procedure is omitted for brevity.

\subsection{Existence of Counterfactual outcome fairness Policy} \label{app:cf}

In this subsection, we briefly discuss how the results in Propositions \ref{prop:vf} and \ref{prop:df} can be extended to counterfactual fairness. For brevity, we omit the proof, as it can be derived in a similar manner to the proofs of Propositions \ref{prop:vf} and \ref{prop:df}.

\begin{assumption} \label{asmp:domination cf}
    Given the subject with information $x$, for any $a\neq a'$, there exists a treatment option $a'$ dominates the other treatment choice $a$, regardless of the sensitive group, i.e., $\min_s f(s,x,a') \geq \max_s f(s,x',a)$.
\end{assumption}

\begin{assumption} \label{asmp:non-domination cf}
    (i) $[f(1,x(1),1)-f(0,x(0),1)][f(1,x(1),0)-f(0,x(0),0)] \leq 0$, for counterfactual covariates $x(1)$ and $x(0)$, almost surely. \\
    (ii) 
    $[f(1,x(1),1)-f(1,x(1),0)][f(0,x(0),1)-f(0,x(0),0)] \leq 0$, for counterfactual covariates $x(1)$ and $x(0)$, almost surely. \\
    (iii) Given subject information $x$, there exists a pair  $(s,a)$ such that $f(s,x,a)\geq f(s,x,a')$, and $f(s,x,a)\geq f(s',x',a)$, for all $x'$, almost surely.
\end{assumption}

The following Propositions \ref{prop:vf cf} and \ref{prop:df cf} establish the sufficient and necessary conditions  for the existence of the outcome fair policy and the double fairness policy, respectively.

\begin{prop} \label{prop:vf cf}
    There exists a counterfactual outcome fair policy for the environment \textbf{if and only if} either Assumption \ref{asmp:domination cf} or Assumption \ref{asmp:non-domination cf}.
\end{prop}

\begin{prop} \label{prop:df cf}
A double counterfactual fairness policy is is either unique or does not exist. Furthermore, a unique double fairness policy exists if and only if Assumption \ref{asmp:non-domination cf} (i) holds.
\end{prop}

\subsection{Extension to Discrete Action Space} \label{app:discrete}

In this subsection, we briefly discuss how the results in Propositions \ref{prop:vf} and \ref{prop:df} can be extended to discrete action spaces. For brevity, we only discuss the extension for the equal opportunity fairness, while the counterfactual fairness can be generalized analogously. 

Throughout this subsection, we assume the action space is $L$-dimensional, and $\mathcal{A}=\{1,2,\dots,L\}.$ Let $\bm{f}(s,x)=\left[f(s,x,1),f(s,x,2),\dots, f(s,x,L)\right]^\top,$ and $\bm{\pi}(s,x)=\left[\pi(1|s,x),\dots, \pi(L|s,x)\right]^\top,$ for any $s$ and $x$. It then follows that  for equal opportunity, we have 
\begin{align*}
    \bm{f}(s,x)^\top \bm{\pi}(s,x)=\bm{f}(s',x)^\top \bm{\pi}(s',x),
\end{align*} for any $s$ and $x$.
Thus, a outcome fair policy exists, if and only if the above linear equation has solutions, subject to $\bm{\pi}^\top \bm 1_L=1$, where $\bm 1_L$ is all-ones vector with dimension $L$;
and $0 \leq \pi(l|s,x) \leq 1$, for any $l=1,2,\dots, L$.

Similarly, a double fairness policy exist if and only if the above linear equation has solutions, subject to $\bm{\pi}^\top \bm 1_L=1$,
and $0 \leq \pi(l|s,x)=\pi(l|s',x) \leq 1$, for any $l=1,2,\dots, L$.

\subsection{Further Details on Premium in the Insurance Data} \label{app:data}

In the offline dataset, we classify customers as low-risk or high-risk based on the BM scale, following common practice in the insurance literature, where drivers are often categorized as good or bad based on their BM score \citep{frangos2001design}. The BM scale ranges from 0 to 22, with higher values indicating greater risk. We define customers with BM greater than five as high-risk and the rest as low-risk.  Since the dataset does not contain premium information, we generate the pure premium amount using the frequency-severity approach, as outlined by \citet{baumann2023fairness}. This approach sets pure premiums based on the predicted risk. To better reflect real-world pricing, we further adjust the final premium based on the coverage plan and gender. Specifically, the premium is generated according to
\begin{gather*}
  \min\left(500 \times \texttt{coverage} + 100 \times \frac{\#\texttt{claims}}{\texttt{exposure}} + 20 \times \texttt{gender},\ 5000\right),
\end{gather*}
where \texttt{coverage} is a categorical variable representing one of three coverage plans, and \texttt{gender} is a binary variable (e.g., 0 for female, 1 for male). The term \#\texttt{claims} refers to the number of insurance claims made by the policyholder, and \texttt{exposure} denotes the duration that the policy was active. The final reward for a subject is defined as the difference between the adjusted premium and the total claim amount.

\section{Proofs and Auxiliary Lemmas}

Throughout, we use $c$ and $C$ to denote a generic constant that can vary from line to line.

\subsection{Proof of Proposition \ref{prop:vf}}

\begin{proof}

It suffices to solve the following linear equation with respect to $\pi$: \begin{align*}
f(0,x,1)\pi(1|0,x)+f(0,x,0)\pi(0|0,x)=f(1,x,1)\pi(1|1,x)+f(1,x,0)\pi(0|1,x).
\end{align*}

However, the above linear systems have two unknowns and only one equation, thus it could have infinity many solutions. Denote $\pi(0,x)$ as $\pi(1|S=0,x)$, this yields that \begin{gather} \label{eq: value eo fair}
    \pi(1,x)=\left[(f(0,x,1)-f(0,x,0))\pi(0,x)+f(0,x,0)-f(1,x,0)\right] [f(1,x,1)-f(1,x,0)]^{-1}.
\end{gather}
It is equivalent to show that there exists a $\pi(0,x)\in [0,1]$ such that its corresponding $\pi(1,x)\in [0,1]$.

    By Equation \eqref{eq: value eo fair}, $\pi(1,x)$ is linear in $\pi(0,x)$. Thus, we have either \begin{align}
        \pi(1,x) \in \left[\frac{f(0,x,0)-f(1,x,0)}{f(1,x,1)-f(1,x,0)}, \frac{f(0,x,1)-f(1,x,0)}{f(1,x,1)-f(1,x,0)}\right], \label{eq:s1}\\
        \mbox{or } \pi(1,x) \in \left[ \frac{f(0,x,1)-f(1,x,0)}{f(1,x,1)-f(1,x,0)},\frac{f(0,x,0)-f(1,x,0)}{f(1,x,1)-f(1,x,0)}\right]. \label{eq:s2}
    \end{align} Thus it remains to show that either the above two intervals intersects with the interval $[0,1]$ is non-empty. By some algebra, interval \eqref{eq:s1} has intersection with $[0,1]$ if and only if the following holds for any $a\neq a'$, \begin{align*} 
    \begin{cases}
        f(1,x,a) > f(1,x,a'),\\
        f(1,x,a) > f(0,x,a'),\\
        f(0,x,a) > f(1,x,a'),\\
        f(0,x,a) > f(0,x,a').
    \end{cases}
    \end{align*} This corresponds to Assumption \ref{asmp:domination}. Following the same argument, interval \eqref{eq:s2} has intersection with $[0,1]$ if and only if for $s\neq s'$ and $a\neq a'$, \begin{align*} 
    \begin{cases}
        f(s,x,a) > f(s,x,a'),\\
        f(s,x,a) > f(s',x,a),\\
        f(s',x,a') > f(s',x,a),\\
        f(s',x,a') > f(s,x,a'),
    \end{cases}
    \end{align*} which corresponds to Assumption \ref{asmp:non-domination}. This thus completes the proof.
\end{proof} 

\subsection{Proof of Proposition \ref{prop:df}}

\begin{proof}

Following the proof of Proposition \ref{prop:vf}, by letting $\pi(1,x)=\pi(0,x)=\pi(x)$, we obtain that \begin{gather} \label{eq: deo policy}
    \pi(x)= [f(0,x,0)-f(1,x,0)] \left[f(0,x,0)-f(1,x,0)+f(1,x,1)-f(0,x,1)\right]^{-1}.
\end{gather} Thus, if $\pi(x) \in [0,1]$, then the solution is unique; and if $\pi(x) \notin [0,1]$, there is no double fairness policy exists.

By \eqref{eq: deo policy}, $\pi(x) \in (0,1)$ if and only if either one of the following two scenarios holds.

\begin{enumerate}
    \item $f(1,x,1) - f(0,x,1) <0$ and $f(0,x,0) -f(1,x,0) <0$;
    \item $f(1,x,1) - f(0,x,1) >0$ and $f(0,x,0) - f(1,x,0) >0$.
\end{enumerate} The above two equations correspond to Assumption \ref{asmp:non-domination} (i). This thus completes the proof.

\end{proof}

\subsection{Proof of Proposition \ref{prop: Cheby}}

\begin{proof}
To show all of the elements in $\Pi_p$ is Pareto optimal solution, see Theorem 5.4.1 in \citet{miettinen1999nonlinear}; while the other direction can be proved using Theorem 5.4.2 in \citet{miettinen1999nonlinear}.
\end{proof}

\subsection{Proof of Theorem \ref{them:regret}}

\begin{lemma} \label{lem:discretize error}
Suppose that the assumptions in Theorem \ref{them:regret} hold, then for any $\alpha \in (0,1)$, there exists some $k\in \left\{1, 2, \cdots, K \right\}$ such that the following holds almost surely:
    \begin{align*}
         \left|M_\alpha^*-\widehat M_{\alpha_k}^*\right| \lesssim \sup_{\pi \in \Pi} \left|\Delta_1(\pi)-\widehat\Delta_1(\pi) \right|+\sup_{\pi \in \Pi} \left|\Delta_2(\pi)-\widehat\Delta_2(\pi)\right|+K^{-1},\\
         \mbox{ and } 
         \left|\Delta_\alpha^*-\widehat \Delta_{\alpha_k}^*\right| \lesssim \sup_{\pi \in \Pi} \left|\Delta_1(\pi)-\widehat\Delta_1(\pi) \right|+\sup_{\pi \in \Pi} \left|\Delta_2(\pi)-\widehat\Delta_2(\pi)\right|+K^{-1}.
    \end{align*} 
\end{lemma}

\begin{proof}
    For any $\alpha \in (0,1)$, there exists $\alpha_k$ such that $\alpha_k \in \Omega$ and $|\alpha_k-\alpha| \leq K^{-1}$. We now first prove the bound for $\big|M_\alpha^*-\widehat M_{\alpha_k}^*\big| $. By definition, we have \begin{gather*}
        M_{\alpha}^*= \inf_{\pi \in \Pi}\max\left\{\alpha\Delta_1(\pi),(1-\alpha)\Delta_2(\pi)\right\}, \\
        \widehat M_{\alpha_k}^*=\inf_{\pi \in \Pi}\max\big\{\alpha_k\widehat\Delta_1(\pi),(1-\alpha_k)\widehat\Delta_2(\pi)\big\}.
        \end{gather*}

It then follows that
        \begin{align} \label{eq:evaluation bound}
   \nonumber &\left|M_\alpha^*-\widehat M_{\alpha_k}^*\right|\\
           \nonumber  =& \left|\inf_{\pi \in \Pi}\left [\max\left(\alpha\Delta_1(\pi),(1-\alpha)\Delta_2(\pi)\right) \right]-\inf_{\pi \in \Pi}\left [\max\left(\alpha_k\widehat\Delta_1(\pi),(1-\alpha_k)\widehat\Delta_2(\pi)\right) \right]\right|\\
            \leq & \; \sup_{\pi \in \Pi} \left|\max\left(\alpha\Delta_1(\pi),(1-\alpha)\Delta_2(\pi)\right)- \max\left(\alpha_k\widehat\Delta_1(\pi),(1-\alpha_k)\widehat\Delta_2(\pi)\right) \right|\\
           \nonumber  \leq & \; \sup_{\pi \in \Pi} \max\Big[\left|\alpha\Delta_1(\pi)-\alpha_k\widehat\Delta_1(\pi)\right|,\left|(1-\alpha)\Delta_2(\pi)-(1-\alpha_k)\widehat\Delta_2(\pi)\right|\Bigg]
             \\
            \nonumber \leq & \; \sup_{\pi \in \Pi} \left|\alpha\Delta_1(\pi)-\alpha_k\widehat\Delta_1(\pi)\right|+
             \sup_{\pi \in \Pi} \left|(1-\alpha)\Delta_2(\pi)-(1-\alpha_k)\widehat\Delta_2(\pi)\right|,
        \end{align} where the second to last inequality follows from triangle inequality of the $\ell_\infty$ norm.
        
We first deal with term $\sup_{\pi \in \Pi} \left|\alpha\Delta_1(\pi)-\alpha_k\widehat\Delta_1(\pi)\right|$, the second term can be bounded similarly. It follows that \begin{align}
    \notag & \sup_{\pi \in \Pi} \left|\alpha\Delta_1(\pi)-\alpha_k\widehat\Delta_1(\pi)\right| = \sup_{\pi \in \Pi} \left|\alpha\Delta_1(\pi)-\alpha_k\Delta_1(\pi)+\alpha_k\Delta_1(\pi)-\alpha_k\widehat\Delta_1(\pi)\right| \\
\leq  &  \; \underbrace{\sup_{\pi \in \Pi} \left|\alpha\Delta_1(\pi)-\alpha_k\Delta_1(\pi)\right|}_{J_1}+ \underbrace{\sup_{\pi \in \Pi} \left|\alpha_k\Delta_1(\pi)-\alpha_k\widehat\Delta_1(\pi)\right|}_{J_2}. \label{eq:discret}
\end{align} Since $\alpha_k \in \Omega$, the first term $J_1$ is upper bounded by $K^{-1}$, by the fact that $\pi(a|s,x)\in [0,1]$ for any $a$, $s$, and $x$, and thus $\sup_{\pi}|\Delta_1(\pi)| \leq 1$. The second term $J_2$ is simply bounded by $\sup_{\pi \in \Pi} \left|\Delta_1(\pi)-\widehat\Delta_1(\pi)\right|$. Thus we have 
proved that \begin{gather*}
    \sup_{\pi \in \Pi} \left|\alpha\Delta_1(\pi)-\alpha_k\widehat\Delta_1(\pi)\right| \leq K^{-1}+\sup_{\pi \in \Pi} \left|\Delta_1(\pi)-\widehat\Delta_1(\pi)\right|.
\end{gather*} Using the similar argument, we can show that \begin{gather*}
    \sup_{\pi \in \Pi} \left|(1-\alpha)\Delta_2(\pi)-(1-\alpha_k)\widehat\Delta_2(\pi)\right| \lesssim K^{-1}+\sup_{\pi \in \Pi} \left|\Delta_2(\pi)-\widehat\Delta_2(\pi)\right|.
\end{gather*}
Finally, the term $\left|\Delta_\alpha^*-\widehat \Delta_{\alpha_k}^*\right| $ can be bounded analogously and we omit the details. This completes the proof.
\end{proof}

We now prove Theorem \ref{them:regret}.
\begin{proof}

We first show if the slack parameter $\kappa$ is chosen such that
\[
\kappa \geq C \max\left[\sup_{\pi \in \Pi} \left| \Delta_1(\pi) - \widehat{\Delta}_1(\pi)\right| + \sup_{\pi \in \Pi} \left|\Delta_2(\pi) - \widehat{\Delta}_2(\pi) \right|, \; K^{-1} \right],
\]
for a sufficiently large constant $C$, then the estimated Pareto set covers the underlying true Pareto set, almost surely. 

Recall that \begin{gather*}
    \Pi_p=\bigcup_{\alpha\in(0,1)} \left\{\pi \in \Lambda_{\alpha}^*: \Delta(\pi) \leq  \Delta_{\alpha}^* \right\}, \\
     \mbox{ and }  \widehat\Pi_p= \bigcup_{k=1}^K \left\{\pi \in \widehat\Lambda_{\alpha_k}^*: \widehat\Delta(\pi) \leq \widehat \Delta_{\alpha_k}^*+\kappa  \right\}.
    \end{gather*}    
 It is therefore equivalent to show that, for any $\alpha \in (0,1)$ and $\pi$: 
    \begin{gather*}
        \Delta(\pi) \leq    \Delta_{\alpha}^* \mbox{ and } M_{\alpha}(\pi) =  M_{\alpha}^*, \\
        \notag \implies \exists \, \alpha_k, \widehat\Delta(\pi) \leq    \widehat \Delta_{\alpha_k}^*+\kappa  \mbox{ and } \widehat M_{\alpha_k}(\pi) \leq  \widehat M_{\alpha_k}^* +\kappa.
    \end{gather*} 
It thus suffices to show that, for any $\alpha \in (0,1)$ and $\pi$: 
\begin{gather}
    \label{eq:Dstar}
    M_{\alpha}(\pi) = M_{\alpha}^* \implies \widehat M_{\alpha_k}(\pi) \leq  \widehat M_{\alpha_k}^* +\kappa,  \\ 
   \mbox{ and } \notag \Delta(\pi) \leq    \Delta_{\alpha}^* \implies \exists \, \alpha_k, \widehat\Delta(\pi) \leq    \widehat \Delta_{\alpha_k}^*+\kappa.
\end{gather}

We first show that $M_{\alpha}(\pi) =  M_{\alpha}^* \implies \widehat M_{\alpha_k}(\pi) \leq  \widehat M_{\alpha_k}^* +\kappa.$
Recall that \begin{gather*}
    \widehat M_{\alpha_k}(\pi)=\max\left(\alpha_k\widehat\Delta_1(\pi),(1-\alpha_k)\widehat \Delta_2(\pi)\right)\mbox{ and }\widehat M_{\alpha_k}^*=\inf_{\pi\in \Pi} \widehat M_{\alpha_k}(\pi).
\end{gather*} 
It follows that
    \begin{align*}
   & \widehat M_{\alpha_k}(\pi)=\max\left(\alpha_k\widehat\Delta_1(\pi),(1-\alpha_k)\widehat\Delta_2(\pi)\right)  \\
    = &  \max\left(\alpha_k\widehat\Delta_1(\pi),(1-\alpha_k)\widehat\Delta_2(\pi)\right) - \max\left(\alpha\Delta_1(\pi),(1-\alpha)\Delta_2(\pi)\right)\\
    &+\max\left(\alpha\Delta_1(\pi),(1-\alpha)\Delta_2(\pi)\right)\\
    \lesssim &  \sup_{\pi \in \Pi} \left| \Delta_1(\pi) - \widehat{\Delta}_1(\pi)\right| + \sup_{\pi \in \Pi} \left|\Delta_2(\pi) - \widehat{\Delta}_2(\pi) \right|+K^{-1} + \max\left(\alpha\Delta_1(\pi),(1-\alpha)\Delta_2(\pi)\right)\\
    \lesssim & \; \sup_{\pi \in \Pi} \left| \Delta_1(\pi) - \widehat{\Delta}_1(\pi)\right| + \sup_{\pi \in \Pi} \left|\Delta_2(\pi) - \widehat{\Delta}_2(\pi) \right|+K^{-1}+M_{\alpha}^* \mbox{ (by Equation \ref{eq:Dstar})}\\
    \lesssim & \; \sup_{\pi \in \Pi} \left| \Delta_1(\pi) - \widehat{\Delta}_1(\pi)\right| + \sup_{\pi \in \Pi} \left|\Delta_2(\pi) - \widehat{\Delta}_2(\pi) \right|+K^{-1}+M_{\alpha}^*-\widehat M_{\alpha_k}^*+\widehat M_{\alpha_k}^*\\
    \lesssim & \; \sup_{\pi \in \Pi} \left| \Delta_1(\pi) - \widehat{\Delta}_1(\pi)\right| + \sup_{\pi \in \Pi} \left|\Delta_2(\pi) - \widehat{\Delta}_2(\pi) \right|+K^{-1}+\widehat M_{\alpha_k}^* \mbox{ by Lemma \ref{lem:discretize error}}.
    \end{align*} 
    
Since $\kappa$ satisfies \[
\kappa \geq C \max\left[\sup_{\pi \in \Pi} \left| \Delta_1(\pi) - \widehat{\Delta}_1(\pi)\right| + \sup_{\pi \in \Pi} \left|\Delta_2(\pi) - \widehat{\Delta}_2(\pi) \right|, \; K^{-1} \right],
\]
for a sufficiently large constant $C$, we have proved Equation \eqref{eq:Dstar}.
Another argument that $\Delta(\pi) \leq    \Delta_{\alpha}^* \implies \exists \, \alpha_k, \;  \widehat\Delta(\pi) \leq    \widehat \Delta_{\alpha_k}^*+\kappa$ can be proved using the similarly reasoning, we omit the details to save the space. 

We next prove the regret bound.
Let $V^*=\sup_{\pi \in \Pi_p}V(\pi) $, i.e., the largest value among the Pareto efficiency set. We first decompose $V^*-V(\widehat\pi)$:
    \begin{align*}
   V^*-V(\widehat\pi)=V^*-\widehat V^*+ \underbrace{\widehat V^* -\widehat V(\widehat\pi)}_{I_1}+\widehat V(\widehat\pi)- V(\widehat\pi),
    \end{align*} 
  where $\widehat V^*$ is the value estimator of $V^*$ using the reward estimator.
    The term $I_1\leq 0$ since $\widehat V(\widehat\pi)=\sup_{\pi\in \widehat \Pi_p} \widehat V(\pi)$; and by Theorem \ref{thm:pset}, the estimated Pareto set contains the optimal policy in the Pareto set,  i.e., $\pi^* \in \widehat \Pi_p$. 
    
    It then follows that 
\begin{align*}
   V^*-V(\widehat\pi)
   \leq & \; V^*-\widehat V^*+\widehat V(\widehat\pi)- V(\widehat\pi)\\
   \leq & \; \sup_{\pi \in \Pi_p}|V({\pi})-\widehat V({\pi})|+\sup_{\pi \in \widehat \Pi_p}|V({\pi})-\widehat V({\pi})|.
    \end{align*} 
The above term can be upper bounded by $2\sup_{\pi \in \widehat \Pi_p}|V({\pi})-\widehat V({\pi})|$, and hence further upper bounded by \begin{gather*}
    2\sup_{\pi \in  \Pi}|V({\pi})-\widehat V({\pi})|.
\end{gather*} Another upper bound can be established in the following way, and we use Hausdorff distance to ``match'' each
$\pi\in\widehat\Pi_p$ with a nearby $\pi'\in \Pi_p$. Recall that $$\eta(\pi) \equiv \widehat{V}(\pi) - V(\pi)$$ denotes the policy evaluation error.
We have 
    \begin{align*}
   &\sup_{\pi \in \widehat \Pi_p}|V({\pi})-\widehat V({\pi})|=\sup_{\pi \in \widehat \Pi_p}|\eta(\pi)|\\
   = & \sup_{\pi \in \widehat \Pi_p} \sup_{\substack{\pi' \in \Pi_p: \\d_{\mathrm{TV}}(\pi,\pi')\leq H_{\mathrm{TV}}(\Pi_p,\widehat \Pi_p)}}|\eta(\pi)-\eta(\pi')+\eta(\pi')| \\
   \leq & \; \sup_{d_{\mathrm{TV}}(\pi,\pi')\leq H_{\mathrm{TV}}(\Pi_p,\widehat \Pi_p)}\left|\eta({\pi})-\eta({\pi'})\right|+\sup_{\pi \in \Pi_p}|\eta({\pi})| \\
   \leq & \; \sup_{d_{\mathrm{TV}}(\pi,\pi')\leq H_{\mathrm{TV}}(\Pi_p,\widehat \Pi_p)} |V({\pi})-V({\pi'})|+\sup_{\pi \in \Pi_p}|\eta({\pi})| \\
    & +
   \sup_{d_{\mathrm{TV}}(\pi,\pi')\leq H_{\mathrm{TV}}(\Pi_p,\widehat \Pi_p)} |\widehat V({\pi})- \widehat V({\pi'})|     \\
   \leq & \; \sup_{d_{\mathrm{TV}}(\pi,\pi')\leq H_{\mathrm{TV}}(\Pi_p,\widehat \Pi_p)} \left|\E \; \sum_a r(S,X,a) (\pi(a|S,X)-\pi'(a|S,X))\right|+\sup_{\pi \in \Pi_p}|\eta({\pi})| \\
   & + \sup_{d_{\mathrm{TV}}(\pi,\pi')\leq H_{\mathrm{TV}}(\Pi_p,\widehat \Pi_p)}   \left| \frac{1}{n}\sum_{i=1}^n \sum_a \widehat r(S_i,X_i,a) (\pi(a|S_i,X_i)-\pi'(a|S_i,X_i))\right|.
    \end{align*} 

We first bound the first term, and the last term can be bounded following the similar logic. The first term 
\begin{align*}
   &\sup_{d_{\mathrm{TV}}(\pi,\pi')\leq H_{\mathrm{TV}}(\Pi_p,\widehat \Pi_p)} \left|\E \; \sum_a r(S,X,a) (\pi(a|S,X)-\pi'(a|S,X))\right| \\
   \leq \; & \sup_{d_{\mathrm{TV}}(\pi,\pi')\leq H_{\mathrm{TV}}(\Pi_p,\widehat \Pi_p)} \sup_{s,x} \sum_a  \left|  r(s,x,a) (\pi(a|s,x)-\pi'(a|s,x))\right| \\
   \leq \; &   \sup_{d_{\mathrm{TV}}(\pi,\pi')\leq H_{\mathrm{TV}}(\Pi_p,\widehat \Pi_p)} \underbrace{\sup_{s,x}   \sum_a |\pi(a|s,x)-\pi'(a|s,x)|}_{2d_{\mathrm{TV}}(\pi,\pi')} \; R_{\max}  \mbox{  (Hölder's inequality
)}\\
   \leq \; & 2R_{\max}H_{\mathrm{TV}}(\Pi_p,\widehat \Pi_p)
\end{align*}

Similarly, we have \begin{gather*}
    \left| \frac{1}{n}\sum_{i=1}^n \sum_a \widehat r(S_i,X_i,a) (\pi(a|S_i,X_i)-\pi'(a|S_i,X_i))\right| \leq 2R_{\max}H_{\mathrm{TV}}(\Pi_p,\widehat \Pi_p).
\end{gather*}


To summarize, we have provided the upper bound for the regret, which is \begin{gather*}
    \min\left[  \sup_{\pi \in \Pi_p} \big| \eta(\pi) \big|+ H_{\mathrm{TV}}(\Pi_p,\widehat{\Pi}_p), \; \sup_{\pi \in \Pi} \big| \eta(\pi) \big| \right] ,
\end{gather*} up to a constant factor.

We next prove the other direction.
\begin{align*}
   &V^*-V({\widehat\pi})=\sup_{\pi \in  \Pi_p}V({\pi})-\sup_{\pi \in \widehat \Pi_p}V({\pi})+\sup_{\pi \in \widehat \Pi_p}V({\pi})-V({\widehat\pi}) \\
   \geq & \; \sup_{\pi \in  \Pi_p}V({\pi})-\sup_{\pi \in \widehat \Pi_p}V({\pi})\\
   \geq & \; -2R_{\max}H_{\mathrm{TV}}(\Pi_p,\widehat \Pi_p),
    \end{align*} where the last inequality follows from \begin{gather*}
        \sup_{\pi \in \widehat \Pi_p}V({\pi}) \leq \sup_{\pi \in \Pi_p}V({\pi})+2R_{\max}H_{\mathrm{TV}}(\Pi_p,\widehat \Pi_p),
    \end{gather*} analogous to the argument we used earlier for the upper bound.
    
    This completes the proof.
\end{proof}

\subsection{Proof of Theorem \ref{thm:pset}}

\begin{proof}
  
In the proof of Theorem \ref{them:regret}, we have demonstrated that  the estimated Pareto set covers the underlying true Pareto set with an appropriate choice of $\kappa$. We next provide the upper bound of the Hausdorff distance $H_{\mathrm{TV}}(\Pi_p,\widehat \Pi_p)$.
Since $\Pi_p\subseteq \widehat \Pi_p$, we have for every $\pi\in\Pi_p$,
$d_{\mathrm{TV}}(\pi,\widehat \Pi_p)=0$ (take $\pi$ itself in
$\widehat \Pi_p$), and hence
$
\sup_{\pi\in \Pi_p} d_{\mathrm{TV}}(\pi,\widehat \Pi_p)=0$.
Therefore,
\[
H_{\mathrm{TV}}(\Pi_p,\widehat \Pi_p)
=\max\big[ \sup_{\pi\in \Pi_p} d_{\mathrm{TV}}(\pi,\widehat \Pi_p), \sup_{\pi\in \widehat \Pi_p} d_{\mathrm{TV}}(\pi,\Pi_p) \big]=
\sup_{\pi\in \widehat \Pi_p} d_{\mathrm{TV}}(\pi,\Pi_p).
\]

Because $\widehat\Pi_p\subseteq\Pi$,
by Assumption \ref{asmp: sep}, to show that $\sup_{\pi\in \widehat \Pi_p} d_{\mathrm{TV}}(\pi,\Pi_p) \leq \delta$, where $\delta$ represents the target bound, it suffices to prove that $\forall \; \widetilde \pi \in \widehat \Pi_p$, $\exists\; \alpha(\widetilde \pi) \in (0,1)$, such that  
\begin{gather*}
  \left|\Delta (\widetilde \pi) -\Delta_{\alpha(\widetilde \pi)}^*\right| \leq c \delta,  \mbox{ and }  \left|M_{\alpha(\widetilde \pi)} (\widetilde\pi) -M_{\alpha}^*\right| \leq c \delta ,
\end{gather*} where $c>0$ is the universal constant in Assumption \ref{asmp: sep}.


We first show that $\forall \; \widetilde \pi \in \widehat \Pi_p$, $\exists\; \alpha(\widetilde \pi) \in (0,1)$, such that $\left|M_{\alpha(\widetilde \pi)} (\widetilde\pi) -M_{\alpha(\widetilde \pi)}^*\right| \leq c \delta$. It follows that
\begin{align*}
   & \left|M_{\alpha(\widetilde \pi)} (\widetilde\pi) -M_{\alpha(\widetilde \pi)}^*\right|=\left|M_{\alpha(\widetilde \pi)} (\widetilde\pi) -\widehat M_{\alpha(\widetilde \pi)} (\widetilde\pi)+\widehat M_{\alpha(\widetilde \pi)} (\widetilde\pi)-
    M_{\alpha(\widetilde \pi)}^* \right|\\
    \leq \; & \underbrace{\left|M_{\alpha(\widetilde \pi)} (\widetilde\pi) -\widehat M_{\alpha(\widetilde \pi)} (\widetilde\pi)\right|}_{Q_1} + \underbrace{\left|\widehat M_{\alpha(\widetilde \pi)} (\widetilde\pi)-
    M_{\alpha(\widetilde \pi)}^*\right|}_{Q_2}.
\end{align*}

We first bound $Q_2$.
Recall that by Equation \eqref{eq:solution constraint}, we have \begin{gather*}
    \widehat \Pi_p= \bigcup_{k=1}^K \left\{\pi \in \widehat\Lambda_{\alpha_k}^*: \widehat\Delta(\pi) \leq \widehat \Delta_{\alpha_k}^*+\kappa\right\} \text{ and }
\widehat\Lambda_{\alpha}^*=\left\{\pi\in \Pi:
\widehat M_{\alpha}(\pi) \leq \widehat M_{\alpha}^*+\kappa
\right\}.
\end{gather*} 
Thus $\forall \; \widetilde \pi \in \widehat \Pi_p$, there exists a $k$, such that 
$\widehat M_{\alpha_k}(\widetilde \pi) \leq \widehat M_{\alpha_k}^*+\kappa$, and we could choose $\alpha(\widetilde \pi)$ to be this $\alpha_k.$ It follows that \begin{align*}
    & Q_2 \leq \underbrace{\left|\widehat M_{\alpha(\widetilde \pi)} (\widetilde\pi)-\widehat M_{\alpha(\widetilde \pi)}^*\right|}_{\leq \kappa}+\left|\widehat M_{\alpha(\widetilde \pi)}^*-
    M_{\alpha(\widetilde \pi)}^*\right| \\
    \leq & \kappa +\sup_{\pi \in \Pi} \left| \Delta_1(\pi) - \widehat{\Delta}_1(\pi)\right| + \sup_{\pi \in \Pi} \left|\Delta_2(\pi) - \widehat{\Delta}_2(\pi) \right|,
\end{align*} where we use Lemma \ref{lem:discretize error} to upper bound $\left|\widehat M_{\alpha(\widetilde \pi)}^*-
    M_{\alpha(\widetilde \pi)}^*\right|.$


It remains to bound term $Q_1$. It follows that \begin{align*}
   &Q_1=\left|M_{\alpha(\widetilde \pi)} (\widetilde\pi) -\widehat M_{\alpha(\widetilde \pi)} (\widetilde\pi)\right|\\
   =& \left|\max\left\{\alpha(\widetilde \pi)\Delta_1(\widetilde \pi),(1-\alpha(\widetilde \pi))\Delta_2(\widetilde \pi)\right\}-\max\left\{\alpha(\widetilde \pi)\widehat\Delta_1(\widetilde \pi),(1-\alpha(\widetilde \pi))\widehat \Delta_2(\widetilde \pi)\right\}\right|\\
   \leq & \; \sup_{\alpha \in (0,1)}\sup_{\pi \in \Pi} \Big|\max\left(\alpha\Delta_1(\pi),(1-\alpha)\Delta_2(\pi)\right)- \max\left(\alpha\widehat\Delta_1(\pi),(1-\alpha)\widehat\Delta_2(\pi)\right) \Big|\\
   \leq & \; \sup_{\alpha \in (0,1)}\sup_{\pi \in \Pi} \Big|\max\Big[\alpha \left|\Delta_1(\pi) - \widehat{\Delta}_1(\pi)\right|, (1-\alpha)\left|\Delta_2(\pi) - \widehat{\Delta}_2(\pi) \right| \Big] \Big|\\
   \leq & \; \sup_{\pi \in \Pi} \left| \Delta_1(\pi) - \widehat{\Delta}_1(\pi)\right| + \sup_{\pi \in \Pi} \left|\Delta_2(\pi) - \widehat{\Delta}_2(\pi) \right|.
\end{align*}

Thus, we have $\forall \; \widetilde \pi \in \widehat \Pi_p$, $\exists\; \alpha(\widetilde \pi) \in (0,1)$, such that \begin{align*}
    & \left|M_{\alpha(\widetilde \pi)} (\widetilde\pi) -M_{\alpha(\widetilde \pi)}^*\right| \\
    \lesssim & \; \sup_{\pi \in \Pi} \left| \Delta_1(\pi) - \widehat{\Delta}_1(\pi)\right| + \sup_{\pi \in \Pi} \left|\Delta_2(\pi) - \widehat{\Delta}_2(\pi) \right|+\kappa
\end{align*}

Similarly, we can show that $\forall \; \widetilde \pi \in \widehat \Pi_p$, $\exists\; \alpha(\widetilde \pi) \in (0,1)$, such that
\begin{gather*}
    \left|\Delta (\widetilde \pi) -\Delta_{\alpha}^*\right|  \lesssim \sup_{\pi \in \Pi} \left| \Delta_1(\pi) - \widehat{\Delta}_1(\pi)\right| + \sup_{\pi \in \Pi} \left|\Delta_2(\pi) - \widehat{\Delta}_2(\pi) \right|+\kappa.
\end{gather*}
Therefore, by Assumption \ref{asmp: sep}, we have 
\begin{align*}
   & H_{\mathrm{TV}}(\Pi_p,\widehat \Pi_p)=\sup_{\pi\in \widehat \Pi_p} d_{\mathrm{TV}}(\pi,\Pi_p) \\
   \lesssim \; & \sup_{\pi \in \Pi} \left| \Delta_1(\pi) - \widehat{\Delta}_1(\pi)\right| + \sup_{\pi \in \Pi} \left|\Delta_2(\pi) - \widehat{\Delta}_2(\pi) \right|+\kappa.
\end{align*}

Finally, recall that \begin{gather*}
    \kappa \geq C \max\Big[\sup_{\pi \in \Pi} \big| \Delta_1(\pi) - \widehat{\Delta}_1(\pi)  \big|+\sup_{\pi \in \Pi} \big| \Delta_2(\pi) - \widehat{\Delta}_2(\pi)  \big|, \; K^{-1} \Big], 
\end{gather*}
for some constant $C$. Thus, we have \begin{gather*}
    H_{\mathrm{TV}}(\Pi_p,\widehat \Pi_p) \lesssim  \kappa.
\end{gather*}
This completes the proof.
\end{proof}

\subsection{Proof of Theorem \ref{thm:eoapp}: Equal Opportunity}

\begin{proof}
    By Theorems~\ref{them:regret} and \ref{thm:pset}, it remains to  
\begin{enumerate}
\item Establish an upper bound for 
    \(\sup_{\pi \in \Pi} \big| \widehat V(\pi) - V(\pi) \big|\);
    \item Establish an upper bound for
    \begin{gather*}
        \sup_{\pi \in \Pi} \big| \Delta_1(\pi) - \widehat{\Delta}_1(\pi)\big| 
        + \sup_{\pi \in \Pi} \big| \Delta_2(\pi) - \widehat{\Delta}_2(\pi) \big|,
    \end{gather*}
    under the \textit{equal opportunity} fairness notion. 
\end{enumerate}

We now present the following lemmas.

\begin{lemma}[Uniform Convergence of the Value Estimator] \label{lemma:value}
    \begin{align*}
        \sup_{\pi \in \Pi} |\widehat V(\pi) - V(\pi) |=O_p\left(\sqrt{\mbox{VC}(\Pi)/n}+n^{-\gamma_1/2}+\sqrt{\log n/n}\right).
    \end{align*} 
\end{lemma}

\begin{proof}
    \begin{align*}
       & \sup_{\pi \in \Pi} |\widehat V(\pi) - V(\pi) |\\
    = & \; \sup_{\pi \in \Pi}  \left|\frac{1}{n}\sum_{i=1}^n \sum_a \widehat r(S_i,X_i,a) \pi(a|S_i,X_i) -\E\sum_a r(S,X,a) \pi(a|S,X)\right|\\
    \leq & \; \sup_{\pi \in \Pi}  \left|\sum_a \frac{1}{n}\sum_{i=1}^n  (\widehat r(S_i,X_i,a)-r(S_i,X_i,a)) \pi(a|S_i,X_i) \right| \\
     +& \; \sup_{\pi \in \Pi}\left|\frac{1}{n}\sum_{i=1}^n r(S_i,X_i,a) \pi(a|S_i,X_i)- \E\sum_a r(S,X,a) \pi(a|S,X)\right|.
    \end{align*} Applying the Cauchy-Schwarz, the first term above is upper bounded by \begin{gather*}
        \sqrt{\max_a \frac{1}{n} \sum_{i=1}^n (\widehat r(S_i,X_i,a)-r(S_i,X_i,a))^2 }=O_p\left(n^{-\frac{\gamma_1}{2}}\right),
    \end{gather*} by Assumption \ref{asmp:nuisance}. Using standard empirical process techniques (see, e.g., \citet{van1996weak}), the second term is of order 
\begin{gather*}
O_p\left(\sqrt{\frac{\mathrm{VC}(\Pi)}{n}} + \sqrt{\frac{\log n}{n}}\right).
\end{gather*}
This completes the proof.
\end{proof} 

\begin{lemma} [Equal Opportunity Action Fairness Evaluation] \label{lemma:eo action fair}
For the equal opportunity fair policy, we have
    \begin{align*}
        \sup_{\pi \in \Pi} |\Delta_1(\pi)-\widehat\Delta_1(\pi)|=O_p\left(\sqrt{\mbox{VC}(\Pi)/n}+\sqrt{\log n/n}\right)
    \end{align*}
\end{lemma}



\begin{proof}
    It suffices to show that \begin{align*}
        \E \sup_{\pi \in \Pi} |\Delta_1(\pi)-\widehat\Delta_1(\pi)|=O\left(\sqrt{\mbox{VC}(\Pi)/n} \right).
    \end{align*} Let $\Delta_{1i}\equiv \left[\pi(1,X_i)-\pi(0,X_i)\right]^2$, 
    and $\{\epsilon_i\}_i$ be i.i.d. Rademacher random variables. 
Because $\Delta_1(\pi)$  is an expectation of i.i.d. copies of $\Delta_{1i}$ and $\widehat\Delta_1(\pi)$ is the corresponding empirical mean, we obtain by applying the symmetrization technique,
\begin{align*}
        &\E \sup_{\pi \in \Pi} |\Delta_1(\pi)-\widehat\Delta_1(\pi)| \leq 2\E \sup_{\pi \in \Pi} \left\lvert\cfrac{1}{n}\sum_{i=1}^n\epsilon_i \Delta_{1i}(\pi)\right\rvert\\
        & = 2\E \sup_{\pi \in \Pi} \cfrac{1}{n}\left\lvert \sum_{i=1}^n \epsilon_i (\pi(1, X_i)-\pi(0, X_i))^2\right\rvert \\
        & \leq 2\E \sup_{\pi \in \Pi} \cfrac{1}{n}\left\lvert\sum_{i=1}^n\epsilon_i \pi^2(1, X_i)\right\rvert +2\E \sup_{\pi \in \Pi} \cfrac{1}{n}\left\lvert \sum_{i=1}^n \epsilon_i \pi^2(0, X_i)\right\rvert \\
        & + 4\E \sup_{\pi \in \Pi} \cfrac{1}{n}\left\lvert \sum_{i=1}^n \epsilon_i \pi(0, X_i) \pi(1, X_i)\right\rvert.
    \end{align*}

Since $\Pi$ is a VC-subgraph class and $\pi(\cdot,x)\in[0,1]$, the transformed
classes
\[
\{\pi^2(1,\cdot):\pi\in\Pi\},\qquad
\{\pi^2(0,\cdot):\pi\in\Pi\},\qquad
\{\pi(1,\cdot)\pi(0,\cdot):\pi\in\Pi\}
\]
are still VC-type, because
squaring is a monotone Lipschitz map on $[0,1]$ and products of bounded VC-type
classes remain VC-type (see Lemma 2.6.20 in \citet{van1996weak}). 
Thus, a direct application of Dudley's entropy integral bound yields that \begin{align*}
        \E \sup_{\pi \in \Pi} \left\lvert\cfrac{1}{n}\sum_{i=1}^n\epsilon_i \Delta_{1i}(\pi)\right\rvert=O(\sqrt{\mbox{VC}(\Pi)/n}).
    \end{align*} 
 Define $
Z(X_1,\ldots,X_n)
:=
\sup_{\pi\in\Pi}\bigl|\Delta_1(\pi)-\widehat\Delta_1(\pi)\bigr|$
Fix an index $i$ and let $\mathbf X=(X_1,\ldots,X_n)$ and
$\mathbf X^{(i)}=(X_1,\ldots,X_i',\ldots,X_n)$ differ only at coordinate $i$. We have
$|Z(\mathbf X)-Z(\mathbf X^{(i)})|\le n^{-1}$.
By bounded difference inequality, for any $t>0$,
\[
\Pr\bigl(Z-\E Z \ge t\bigr)
\le \exp(-2nt^2),
\]
Taking $t=\sqrt{\log n/(2n)}$ gives
\[
Z \le \E Z + \sqrt{\frac{\log n}{2n}}
\quad\text{with probability at least } 1-\frac1n.
\]
Combining with $\E Z=O(\sqrt{\mbox{VC}(\Pi)/n})$ yields
\begin{align*}
       \sup_{\pi \in \Pi} |\Delta_1(\pi)-\widehat\Delta_1(\pi)|=O_p\left(\sqrt{\mbox{VC}(\Pi)/n}+\sqrt{\log n/n}\right).
    \end{align*}
    This completes the proof.
\end{proof}

\begin{lemma} [Equal Opportunity Outcome fairness Evaluation] \label{lemma:eo outcome fair}
For the equal opportunity fair policy, the following holds:
    \begin{align*}
        \sup_{\pi \in \Pi} |\Delta_2(\pi)-\widehat\Delta_2(\pi)|=O_p\left(\sqrt{\mbox{VC}(\Pi)/n}+n^{-\gamma_2/2}+\sqrt{\log n/n}\right).
    \end{align*}
\end{lemma}

\begin{proof} By the definition, we have \begin{align*}
    & \sup_{\pi \in \Pi}|\Delta_2(\pi)-\widehat\Delta_2(\pi)|\\
    =&
    \sup_{\pi \in \Pi} \left|\frac{1}{n}\sum_{i=1}^n (\widehat f^\pi(1,X_i)-\widehat f^\pi(0,X_i))^2-\E\left[ (f^\pi(1,X)-f^\pi(0,X) )^2\right] \right| \\
    \leq & \sup_{\pi \in \Pi}\frac{1}{n}\sum_{i=1}^n \left(\widehat f^\pi(1,X_i)-f^\pi(1,X_i)\right)^2
    +\sup_{\pi \in \Pi}\frac{1}{n}\sum_{i=1}^n \left(f^\pi(0,X_i)-\widehat f^\pi(0,X_i)\right)^2\\
    +& \sup_{\pi \in \Pi}2\left|\frac{1}{n}\sum_{i=1}^n \left(f^\pi(1,X_i)-f^\pi(0,X_i)\right)\left(f^\pi(0,X_i)-\widehat f^\pi(0,X_i)\right) \right|\\
    +& \sup_{\pi \in \Pi}2\left|\frac{1}{n}\sum_{i=1}^n \left(f^\pi(1,X_i)-f^\pi(0,X_i)\right)\left(f^\pi(1,X_i)-\widehat f^\pi(1,X_i)\right) \right|\\
    +& \sup_{\pi \in \Pi}2\left|\frac{1}{n}\sum_{i=1}^n \left(f^\pi(0,X_i)-\widehat f^\pi(0,X_i)\right)\left(f^\pi(1,X_i)-\widehat f^\pi(1,X_i)\right) \right|\\
    +&\sup_{\pi \in \Pi}\left|\frac{1}{n}\sum_{i=1}^n \left(f^\pi(1,X_i)-f^\pi(0,X_i)\right)^2-\E\left[ (f^\pi(1,X)-f^\pi(0,X) )^2\right] \right| .
\end{align*} In the above 6 terms, the first five terms are relatively easy to deal with, since their upper bounds can be shown not depend on $\pi$; while the last term $$\sup_{\pi \in \Pi}\left|\frac{1}{n}\sum_{i=1}^n \left(f^\pi(1,X_i)-f^\pi(0,X_i)\right)^2-\E\left[ (f^\pi(1,X)-f^\pi(0,X) )^2\right] \right|$$ needs to be treated using empirical process theory.

We next bound the first term $\sup_{\pi \in \Pi}\frac{1}{n}\sum_{i=1}^n \left(\widehat f^\pi(1,X_i)-f^\pi(1,X_i)\right)^2$, and the second term can be similarly bounded. 
\begin{align*}
    &\sup_{\pi \in \Pi}\frac{1}{n}\sum_{i=1}^n \left(\widehat f^\pi(1,X_i)-f^\pi(1,X_i)\right)^2\\
    =&\sup_{\pi \in \Pi}\frac{1}{n}\sum_{i=1}^n \left(\sum_a \pi(a|1,X_i) \left(\widehat r(1,X_i,a)-r(1,X_i,a)\right)\right)^2 
    \\ \leq &   \sup_{\pi \in \Pi}\frac{1}{n}\sum_{i=1}^n \sum_a \pi(a|1,X_i) \left(\widehat r(1,X_i,a)-r(1,X_i,a)\right)^2 \quad  \text{(Jensen's Inequality)}\\ \leq &   \sup_{\pi \in \Pi} \frac{1}{n} \sum_a \sum_{i=1}^n  \pi(a|1,X_i) \left(\widehat r(1,X_i,a)-r(1,X_i,a)\right)^2 \\
     \leq &   \max_a \frac{1}{n} \sum_{i=1}^n  \left[\widehat r(1,X_i,a)-r(1,X_i,a)\right]^2 = O_p(n^{-\gamma_2}),
\end{align*} where the last equality follows from Assumption \ref{asmp:nuisance}.

We next upper bound the third term. By Cauchy-Schwarz inequality,
\begin{align*}
    &\sup_{\pi \in \Pi}2\left|\frac{1}{n}\sum_{i=1}^n \left(f^\pi(1,X_i)-f^\pi(0,X_i)\right)\left(f^\pi(0,X_i)-\widehat f^\pi(0,X_i)\right) \right| \\
    \lesssim \; & \sup_{\pi \in \Pi} \sqrt{\frac{1}{n}\sum_{i=1}^n \left[f^\pi(0,X_i)-\widehat f^\pi(0,X_i) \right]^2} \sqrt{\frac{1}{n}\sum_{i=1}^n \left[f^\pi(1,X_i)-f^\pi(0,X_i) \right]^2}  \\
    \lesssim \; &\sqrt{\sup_{\pi \in \Pi}  \frac{1}{n}\sum_{i=1}^n \left[f^\pi(0,X_i)-\widehat f^\pi(0,X_i) \right]^2},  
\end{align*} 
where the last inequality is due to the bounded reward assumption. Note that
the above upper bound is simply the square root of the second term,  hence is $O_p(n^{-\gamma_2/2})$.We can bound the 4th and 5th terms similarly. We omit the details to save the space.

It now remains to bound the last term. Define $$\widetilde\Delta_2(\pi) \equiv \frac{1}{n}\sum_{i=1}^n (f^\pi(1,X_i)-f^\pi(0,X_i))^2,$$ i.e., the oracle estimator of $\Delta_2(\pi)$,
and let $\Delta_{2i}=(f^\pi(1,X_i)-f^\pi(0,X_i))^2$.
Applying the symmetrization technique \citep{van1996weak}, we obtain \begin{align*}
        &\E \sup_{\pi \in \Pi} |\Delta_2(\pi)-\widetilde\Delta_2(\pi)| \leq 2\E \sup_{\pi \in \Pi} \left\lvert\cfrac{1}{n}\sum_{i=1}^n\epsilon_i \Delta_{2i}(\pi)\right\rvert\\
         = &2\E \sup_{\pi \in \Pi} \cfrac{1}{n}\left\lvert \sum_{i=1}^n \epsilon_i (f^\pi(1,X_i)-f^\pi(0,X_i))^2\right\rvert\\
         = &2 \E \sup_{\pi \in \Pi} \cfrac{1}{n}\left\lvert \sum_{i=1}^n \epsilon_i \left[\sum_a r(1,X_i,a) \pi(a|1,X_i)-\sum_a r(0,X_i,a) \pi(a|0,X_i)\right]^2\right\rvert\\
         \leq &2 \underbrace{\E \sup_{\pi \in \Pi} \cfrac{1}{n}\left\lvert\sum_{i=1}^n\epsilon_i \left[\sum_a r(1,X_i,a) \pi(a|1,X_i) \right]^2\right\rvert}_{E_1} +2\underbrace{\E \sup_{\pi \in \Pi} \cfrac{1}{n}\left\lvert \sum_{i=1}^n \epsilon_i \left[\sum_a r(1,X_i,a) \pi(a|1,X_i)\right]^2\right\rvert}_{E_2} \\
        & + 4\underbrace{\E \sup_{\pi \in \Pi} \cfrac{1}{n}\left\lvert \sum_{i=1}^n \epsilon_i \sum_a r(1,X_i,a) \pi(a|1,X_i) \sum_a r(1,X_i,a) \pi(a|1,X_i) \right\rvert }_{E_3}. 
    \end{align*}
    It follows that \begin{align*}
        &E_1= \E \sup_{\pi \in \Pi} \cfrac{1}{n}\left\lvert\sum_{i=1}^n\epsilon_i \left[\sum_a r(1,X_i,a) \pi(a|1,X_i) \right]^2\right\rvert\\
        = \; & \E \sup_{\pi \in \Pi} \cfrac{1}{n}\left\lvert\sum_{i=1}^n\epsilon_i \left[ r(1,X_i,1) \pi(1|1,X_i) \right]^2\right\rvert+\E \sup_{\pi \in \Pi} \cfrac{1}{n}\left\lvert\sum_{i=1}^n\epsilon_i \left[ r(1,X_i,0) \pi(0|1,X_i) \right]^2\right\rvert\\
        & +\E \sup_{\pi \in \Pi} \cfrac{1}{n}\left\lvert\sum_{i=1}^n\epsilon_i  r(1,X_i,0) \pi(0|1,X_i)  r(1,X_i,1) \pi(1|1,X_i)\right\rvert.
    \end{align*} By the bounded reward and finite VC subgraph dimension assumption (Assumption \ref{asmp: vc}) and arguing similarly as in the proof of Lemma~\ref{lemma:eo action fair}, the above three terms all are order of $O(\sqrt{\mbox{VC}(\Pi)/n})$. The terms $E_2$ and $E_3$ can be proved analogously. 
    Applying the bounded difference inequality as in the proof of Lemma~\ref{lemma:eo action fair}, we have proved that 
    \begin{align*}
        \sup_{\pi \in \Pi} |\Delta_2(\pi)-\widehat\Delta_2(\pi)|=O_p\left(\sqrt{\mbox{VC}(\Pi)/n}+n^{-\gamma_2/2}+\sqrt{\log n/n}\right)
    \end{align*} for the equal opportunity fairness notion.
\end{proof}
\end{proof}

\subsection{Proof of Theorem \ref{thm:eoapp}: Counterfactual Fairness}

By Theorems~\ref{them:regret} and \ref{thm:pset}, it remains to  
\begin{enumerate}
\item Establish an upper bound for 
    \(\sup_{\pi \in \Pi} \big| \widehat V(\pi) - V(\pi) \big|\);
    \item Establish an upper bound for
    \begin{gather*}
        \sup_{\pi \in \Pi} \big| \Delta_1(\pi) - \widehat{\Delta}_1(\pi)\big| 
        + \sup_{\pi \in \Pi} \big| \Delta_2(\pi) - \widehat{\Delta}_2(\pi) \big|,
    \end{gather*}
    under the \textit{counterfactual} fairness notion.
\end{enumerate}

We present the following lemmas.

\begin{lemma} [Counterfactual Action Fairness Evaluation] \label{lemma:cf action fair}
For the equal opportunity fairness notion, the following holds:
    \begin{align*}
        \sup_{\pi \in \Pi} |\widehat\Delta_1(\pi)-\Delta_1(\pi)|=O_p\left(\sqrt{\mbox{VC}(\Pi)/n}+\sqrt{\log n/n}+ \sqrt{d/n}\right).
    \end{align*}
\end{lemma}

\begin{proof} By the definition, we have 

\begin{align*}
         &\sup_{\pi \in \Pi} |\Delta_1(\pi)-\widehat\Delta_1(\pi)|=\sup_{\pi \in \Pi} \left|\frac{1}{n}\sum_{i=1}^n \left[\pi(S_i,X_i)-\pi(S_i',\widehat X_i')\right]^2 -\E\left[\pi(S,X)-\pi(S',X')\right]^2  \right |\\
         \leq &\; \sup_{\pi \in \Pi} \left|\frac{1}{n}\sum_{i=1}^n \left[\pi(S_i,X_i)-\pi(S_i', \widehat X_i')\right]^2 -\E\left[\pi(S,X)-\pi(S',\widehat X')\right]^2  \right |\\
          & +\sup_{\pi \in \Pi} \left|\E\left[\pi(S,X)-\pi(S',\widehat X')\right]^2 -\E\left[\pi(S,X)-\pi(S', X')\right]^2  \right | \\
         \leq &\;
    \underbrace{\sup_{\substack{ \pi \in \Pi \\ t: \, \norm{t}_2 \leq \delta }} \left|\frac{1}{n}\sum_{i=1}^n \left[\pi(S_i,X_i)-\pi(S_i',  X_i'+t)\right]^2 -\E\left[\pi(S,X)-\pi(S',X'+t)\right]^2  \right | }_{I_1}\\
          & +\underbrace{   \sup_{\substack{ \pi \in \Pi \\ t: \, \norm{t}_2 \leq \delta }}  \left|\E\left[\pi(S,X)-\pi(S', X'+t)\right]^2 -\E\left[\pi(S,X)-\pi(S', X')\right]^2  \right | }_{I_2},
    \end{align*} 
    with high probability $\delta\leq c \sqrt{d/n}$, since $\norm{\widehat X'-X'}_2=O_p(\sqrt{d /n})$. This follows directly from Assumption \ref{asmp: additive error}, allowing $\widehat{\theta}(s)$ for $s = 0, 1$ to be estimated using the sample average,  i.e., $\widehat \theta (s)= \frac{1}{n}\sum_{i=1}^n  \mathds 1(S_i=s) X_i $. 

The second term $I_2=O_p(\sqrt{d/n})$ by Assumption \ref{asmp:lip}. It remains to deal with the empirical process term $I_1.$ Define the function class \begin{align*}
    \Pi_t=\{\pi(a|s,x+t): \pi \in \Pi, \; t \leq \delta  \}.
\end{align*} 
Since $\Pi$ is a VC class, and $t$ ranges over a bounded set, the $\varepsilon$-covering number of the function class  is upper bounded by \begin{gather*}
   C\left(\frac{1}{\varepsilon}\right)^{{\mbox{VC}(\Pi)}} \left(\frac{\delta}{\varepsilon} \right)^d,
\end{gather*} and the corresponding entropy is upper bounded by \begin{gather*}
    d\log \delta+ d\log\left(\frac{1}{\varepsilon} \right)+ \mbox{VC}(\Pi) \log\left(\frac{1}{\varepsilon} \right),
\end{gather*} up to a constant. Thus, the transformed class $\{\pi^2(a|s,x+t): \pi \in \Pi, \; t \leq \delta  \}$ also has bounded entropy (see the argument in the proof of Lemma \ref{lemma:eo action fair}). Therefore, using symmetrization, Dudley’s entropy bound, and bounded difference inequality yields 
\begin{gather*}
    I_1=O_p\left(\sqrt{\mbox{VC}(\Pi)/n}+\sqrt{\log n/n}+ \sqrt{d/n}\right).
\end{gather*}
This completes the proof.
\end{proof}

\begin{lemma} [Counterfactual Outcome fairness Evaluation] \label{lemma:cf outcome fair}
For the counterfactual fairness notion, the following holds:
    \begin{align*}
        \sup_{\pi \in \Pi} |\Delta_2(\pi)-\widehat\Delta_2(\pi)|=O_p\left(\sqrt{\mbox{VC}(\Pi)/n}+n^{-\gamma_2/2}+\sqrt{\log n/n}+ \sqrt{d/n}\right).
    \end{align*}
\end{lemma}

\begin{proof}
    The proof can be established in a manner similar to the proofs of Lemmas \ref{lemma:eo outcome fair} and \ref{lemma:cf action fair}. For brevity, the details are omitted.
\end{proof}
\end{appendices}
\end{document}